\documentclass{article}
\usepackage{ijcai20}

\usepackage{times}

\usepackage{soul}
\usepackage{url}
\usepackage[hidelinks]{hyperref}
\usepackage[utf8]{inputenc}
\usepackage[small]{caption}
\usepackage{graphicx}
\usepackage{amsmath}
\usepackage{amssymb}
\usepackage{booktabs}
\usepackage{enumitem}
\usepackage{algorithm, algorithmic}
\usepackage{multirow}
\usepackage{subfig}
\usepackage{epsfig}

\newtheorem{theorem}{Theorem}
\newtheorem{remark}{Remark}
\newtheorem{lemma}{Lemma}
\newtheorem{corollary}{Corollary}
\newtheorem{proof}{Proof}[section]

\title{Thompson Sampling for Combinatorial Semi-bandits with Sleeping Arms and Long-Term Fairness Constraints}

\author{
Zhiming Huang$^1$\footnote{Contact Author}\and
Yifan Xu$^2$\and
Bingshan Hu$^{1}$\and
Qipeng Wang$^3$\And
Jianping Pan$^1$\\
\affiliations
$^1$Department of Computer Science, University of Victoria, Canada\\
$^2$School of Cyber Science and Engineering, Southeast University, China\\
$^3$School of Electronic and Information Engineering, Harbin Institute of Technology, China\\
}

\begin{document}

\maketitle
\begin{abstract}
We study the combinatorial sleeping multi-armed semi-bandit problem with long-term fairness constraints~(CSMAB-F).
To address the problem, we adopt Thompson Sampling~(TS) to maximize the total rewards and use virtual queue techniques to handle the fairness constraints, and design an algorithm called \emph{TS with beta priors and Bernoulli likelihoods for CSMAB-F~(TSCSF-B)}.
Further, we prove TSCSF-B can satisfy the fairness constraints, and the time-averaged regret is upper bounded by $\frac{N}{2\eta} + O\left(\frac{\sqrt{mNT\ln T}}{T}\right)$, where $N$ is the total number of arms, $m$ is the maximum number of arms that can be pulled simultaneously in each round~(the cardinality constraint) and $\eta$ is the parameter trading off fairness for rewards. By relaxing the fairness constraints (i.e., let $\eta \rightarrow \infty$), the bound boils down to the first problem-independent bound of TS algorithms for combinatorial sleeping multi-armed semi-bandit problems.
Finally, we perform numerical experiments and use a high-rating movie recommendation application to show the effectiveness and efficiency of the proposed algorithm.
\end{abstract}
\section{Introduction}

In this paper, we focus on a recent variant of \emph{multi-armed bandit~(MAB)} problems, which is the combinatorial sleeping MAB with long-term fairness constraints (CSMAB-F)~\cite{li2019combinatorial}. In CSMAB-F, a learning agent needs to simultaneously pull a subset of available arms subject to some constraints~(usually the cardinality constraint) and only observes the reward of each pulled arm~(semi-bandit setting) in each round. Both the availability and the reward of each arm are stochastically generated, and the long-term fairness among arms is further considered, i.e., each arm should be pulled at least a number of times in a long horizon of time. The objective is to accumulate as many rewards as possible in the finite time horizon. 
The CSMAB-F problem has a wide range of real-world applications. For example, in task assignment problems, we want each worker to be assigned for a certain number of tasks~(i.e., fairness constraints), while some of the workers may be unavailable in some time slots~(i.e., sleeping arms). In movie recommendation systems considering movie diversity, different movie genres should be recommended for a certain number of times~(i.e., fairness constraints), while we do not recommend users with genres they dislike~(i.e., sleeping arms).

\emph{Upper Confidence Bound (UCB)} and \emph{Thompson Sampling (TS)} are two well-known families of algorithms to address the stochastic MAB problems. Theoretically, TS is comparable to UCB~\cite{hu2019problem,Agrawal2017}, but practically, TS usually outperforms UCB-based algorithms significantly~\cite{chapelle2011empirical}. However, while the theoretical performance of UCB-based algorithms has been extensively studied for various MAB problems~\cite{bubeck2012regret}, there are only a few theoretical results for TS-based algorithms~\cite{Agrawal2017,Chatterjee2017AnalysisOT,wang2018thompson}.

In ~\cite{li2019combinatorial}, a UCB-based algorithm called \emph{Learning with Fairness Guarantee~(LFG)} was devised and a problem-independent regret bound \footnote{If a regret bound depends on a specific problem instance, we call it a \emph{problem-dependent} regret bound while if a regret bound does not depend on any problem instances, we call it a \emph{problem-independent} regret bound.} $\frac{N}{2\eta} + \frac{2\sqrt{6mNT\ln T}+ 5.11w_{\max}N }{T}$ was derived for the CSMAB-F problem, where $N$ is the number of arms, $m$ is the maximal number of arms that can be pulled simultaneously in each round, $w_{\max}$ is the maximum arm weight,  and  $\eta$ is a parameter that balances the the fairness and the reward. However, as TS-based algorithms are usually comparable to UCB theoretically but practically perform better than UCB, we are motivated to devise TS-based algorithms and derive regret bounds of such algorithms for the CSMAB-F problem.
The contributions of this paper can be summarized as follows. 
\begin{itemize}[leftmargin=*]
    \item We devise the first TS-based algorithm for CSMAB-F problems with a provable upper regret bound. To be fully comparable with LFG, we incorporate the virtual queue techniques defined in~\cite{li2019combinatorial} but make a modification on the queue evolution process to reduce the accumulated rounding errors.
    \item Our regret bound~$\frac{N}{2\eta} + \frac{4\sqrt{mNT\ln T}+ 2.51w_{\max}N }{T}$ is in the same polynomial order as the one achieved by LFG, but with lower coefficients. This fact shows again that TS-based algorithms can achieve comparable theoretical guarantee as UCB-based algorithms but with a tighter bound.
    
    
    \item We verify and validate the practical performance of our proposed algorithms by numerical experiments and real-world applications. Compared with LFG, it is shown that TSCSF-B does perform better than LFG in practice.
\end{itemize}
  It is noteworthy that our algorithmic framework and proof techniques are extensible to other MAB problems with other fairness definitions. Furthermore, if we do not consider the fairness constraints, our bound boils down to the first problem-independent upper regret bounds of TS algorithms for CSMAB problems and matches the lower regret bound~\cite{bubeck2012regret}.



The remainder of this paper is organized as follows. In Section~\ref{sec:related}, we summarize the most related works about CSMAB-F. The problem formulation of CSMAB-F is presented in Section~\ref{sec:problem}, following what in \cite{li2019combinatorial} for comparison purposes. The proposed TS-based algorithm is presented in Section~\ref{sec:algorithms}, with main results, i.e., the fairness guarantee, performance bounds and proof sketches, presented in Section~\ref{sec:results}. 
Performance evaluations are presented in Section~\ref{sec:evaluation}, followed by concluding remarks and future work in Section~\ref{sec:conclusion}. Detailed proofs can be found in Appendix \ref{sec:appendix}.
\section{Related Works}\label{sec:related}

Many variants of the stochastic MAB problems have been proposed and the corresponding regret bounds have been derived. The ones that are most related to our work are \emph{Combinatorial MAB (CMAB)} and its variants, which was first proposed and analyzed in \cite{gai2012combinatorial}. In CMAB, an agent needs to pull a combination of arms simultaneously from a fixed arm set. Considering a semi-bandit feedback setting, i.e., the individual reward of each arm in the played combinatorial action can be observed, the authors of \cite{chen2013combinatorial} derived a sublinear problem-dependent upper regret bound based on a UCB algorithm and this bound was further improved in \cite{kveton2015tight}. In \cite{combes2015combinatorial}, a problem-dependent lower regret bound was derived by constructing some problem instances. Very recently, the authors of \cite{wang2018thompson} derived a problem-dependent regret bound of TS-based algorithms for CMAB problems. 

All the aforementioned works make an assumption that the arm set from which the learning agent can pull arms is fixed over all $T$ rounds, i.e., all the arms are always available and ready to be pulled. However, in practice, some of the arms may not be available in some rounds, for example, some items for online sales are out of stock temporarily. Therefore, a bunch of literature studied the setting of \emph{MAB with sleeping arms (SMAB)} ~\cite{kleinberg2010regret,Chatterjee2017AnalysisOT,hu2019int,NIPS2016_6450,NIPS2014_5381}.  In the SMAB setting, the set of available arms for each round, i.e., the availability set, can vary. For the simplest version of SMAB (only one arm is pulled in each round), the problem-dependent regret bounds of UCB-based algorithms and TS-based algorithms have been analyzed in \cite{kleinberg2010regret} and \cite{Chatterjee2017AnalysisOT}, respectively.

Regarding the combinatorial SMAB setting~(CSMAB), some negative results are shown in \cite{NIPS2016_6450}, i.e., efficient no-regret learning algorithms sometimes are computationally hard. However, for some settings such as stochastic availability and stochastic reward,  it is shown that it is still possible to devise efficient learning algorithms with good theoretical guarantees \cite{hu2019int,li2019combinatorial}. More importantly, in the work of  \cite{li2019combinatorial}, they considered a new variant called the \emph{combinatorial sleeping MAB with long-term fairness constraints (CSMAB-F)}. In this setting, fairness among arms is further considered, i.e., each arm needs to be pulled for a number of times. The authors designed a UCB-based algorithm called \emph{Learning with Fairness Guarantee~(LFG)} and provided a problem-independent time-averaged upper regret bound.

Due to the attractive practical performance and lack of theoretical guarantees for TS-based algorithms in CSMAB-F, it is desirable to devise a TS-based algorithm and derive regret bounds for such algorithms. We are interested to derive the problem-independent regret bound as it holds for all problem instances. In this work, we give the first provable regret bound that is in the same polynomial order as the one in \cite{li2019combinatorial} but with lower coefficients.  To the best of our knowledge, the derived upper bound is also the first problem-independent regret bound of TS-based algorithms for CSMAB problems which matches the lower regret bounds~\cite{kveton2015tight} when relaxing the long-term fairness constraints.

\section{Problem Formulation}\label{sec:problem}
In this section, we present the problem formulation of the stochastic combinatorial sleeping multi-armed bandit problem with fairness constraints~(CSMAB-F), following ~\cite{li2019combinatorial} closely for comparison purposes. To state the problem clearly, we first introduce the CSMAB problem and then incorporate the fairness constraints.

Let set $\mathcal{N} := \{1,2,\ldots,N\}$ be an arm set and $\Theta := 2^{\mathcal{N}}$ be the power set of $\mathcal{N}$. At the beginning of each round $t = 0, 1, \ldots, T-1$, a set of arms $Z(t) \in \Theta$ are revealed to a learning agent according to a fixed but unknown distribution $P_{Z}$ over $\Theta$, i.e., $P_{Z}: \Theta \rightarrow [0,1]$. We call set $Z(t)$ the \emph{availability set in round $t$}. Meanwhile, each arm $i \in \mathcal{N}$ is associated with a random reward $X_i(t) \in \{0,1\}$ drawn from a fixed Bernoulli distribution $D_i$ with an unknown mean $u_i := \mathbb{E}_{X_i(t) \sim D_i} \left[ X_i \right]$~\footnote{Note that we only consider Bernoulli distribution in this paper for brevity, but it is feasible to extend our algorithm and analysis with little modifications to other general distributions~(see \cite{agrawal2012analysis,Agrawal2017}).}, and a fixed known non-negative weight $w_i$ for that arm. Note that for all the arms in $\mathcal{N}$, their rewards are drawn independently in each round $t$.
Then the learning agent pulls a subset of arms $A(t)$ from the availability set with the cardinality no more than $m$, i.e., $A(t) \subseteq Z(t), |A(t)|\leq m$, and receives a weighted random reward $R(t) := \sum_{i \in A(t)} w_i X_i(t)$. 

In this work,  we consider the semi-bandit feedback setting, which is consistent with \cite{li2019combinatorial}, i.e., the learning agent can observe the individual random reward of all the arms in $A(t)$. Note that since the availability set $Z(t)$ is drawn from a fixed distribution $P_{Z}$ and the random rewards of the arms are also drawn from fixed distributions, we are in a bandit setting called the \emph{stochastic availability} and the \emph{stochastic reward}. 


The objective of the learning agent is to pull the arms sequentially to maximize the expected time-averaged rewards over $T$ rounds, i.e., $\max \mathbb{E}\left[\frac{1}{T}\sum_{t=0}^{T-1} R(t)\right] $. 

Furthermore, we consider the long-term fairness constraints proposed in \cite{li2019combinatorial}, where 
each arm $i \in \mathcal{N}$ is expected to be pulled at least $k_i \cdot T$ times when the time horizon is long enough, i.e.,
\begin{equation}\label{eq:fraction}
    \liminf _{T \rightarrow \infty} \frac{1}{T} \sum_{t=0}^{T-1} \mathbb{E}\left[\mathbf{1}[i \in A(t)]\right] \geq k_{i}, \forall i \in \mathcal{N}.
\end{equation}
We say a vector $\mathbf{k} := \begin{bmatrix}
k_1 & k_2 & \cdots & k_N
\end{bmatrix}^T$ is feasible if there exists a policy such that (\ref{eq:fraction}) is satisfied.

If we knew the availability set distribution $P_{Z}$ and the mean reward $u_i$ for each arm $i$ in advance, and $\mathbf{k}$ was feasible, then there would be a randomized algorithm which is the optimal solution for CSMAB-F problems. The algorithm chooses arms $A(t)\subseteq S$ with probability $q_S(A)$ when observing available arms $S \in \Theta$.  Let $\mathbf{q}:=\{q_S(A), \forall S \in \Theta, \forall A \subseteq S : |A| \leq m \}$. We can determine $\mathbf{q}$ by solving the following problem:
\begin{equation}\label{eq:question}
\resizebox{.9\hsize}{!}{$
    \begin{array}{ll}{
    \underset{\mathbf{q}}{\operatorname{maximize}}} & {\sum\limits_{S \in \Theta} P_{Z}(S) \sum\limits_{A \subseteq S, |A|\leq m} q_{S}(A) \sum\limits_{i \in A} w_{i} u_{i}} \\ 
    {\text { subject to }} & {\sum\limits_{S \in \Theta} P_{Z}(S) \sum\limits_{A \subseteq S, |A|\leq m : i \in A} q_{S}(A) \geq k_{i}, \forall i \in \mathcal{N}}, \\  
    {}& {\sum\limits_{A \subseteq S : |A|\leq m} q_{S}(A)=1, \forall S \in \Theta}, \\ 
    {} & {q_{S}(A) \in[0,1], \forall A \subseteq S, |A|\leq m, \forall S \in \Theta},
    \end{array}
$}
\end{equation}
where the first constraint is equivalent to the fairness constraints defined in (\ref{eq:fraction}), and the second constraint states that for each availability set $S\in \Theta$, the probability space for choosing $A(t)$ should be complete.

Denote the optimal solution to (\ref{eq:question}) as $\mathbf{q}^* = \{q^*_{S}(A), \forall S \in \Theta, A \subseteq S, |A|\leq m\}$, i.e., the optimal policy pulls $A\subseteq S$ with probability $q^*_S(A)$ when observing an available arm set $S$. We denote by $A^*(t)$ the arms pulled by the optimal policy in round $t$.

However,  $P_{Z}$ and $u_i, \forall i \in \mathcal{N}$ are unknown in advance, and the learning agent can only observe the available arms and the random rewards of the pulled arms. Therefore, the learning agent faces the dilemma between exploration and exploitation, i.e., in each round, the agent can either do exploration (acquiring information to estimate the mean reward of each arm) or exploitation (accumulating rewards as many as possible). The quality of the agent’s policy is measured by the \emph{time-averaged regret}, which is the performance loss caused by not always performing the optimal actions.
Considering the stochastic availability of each arm, we define the time-averaged regret as follows:
\begin{equation}\label{eq:regret}
\resizebox{.9\hsize}{!}{$\mathcal{R}(T) := \mathbb{E}\left[\frac{1}{T} \sum_{t=0}^{T-1} \left(\sum_{i \in A^*(t)} w_{i} X_{i}(t) - \sum_{i \in A(t)} w_{i} X_{i}(t)\right)\right]$.}
\end{equation}
\section{Thompson Sampling with Beta Priors and Bernoulli Likelihoods for CSMAB-F (TSCSF-B)}\label{sec:algorithms}


The key challenges to design an effective and efficient algorithm to solve the CSMAB-F problem can be twofold. First, the algorithm should well balance the exploration and exploitation in order to achieve a low time-averaged regret. Second, the algorithm should make a good balance between satisfying the fairness constraints and accumulating more rewards.

\begin{algorithm}[h]
\caption{Thompson Sampling with Beta Priors and Bernoulli Likelihoods for CSMAB-F (TSCSF-B)} \label{alg:CombTS1}
\begin{algorithmic}[1]
  \REQUIRE Arm set $\mathcal{N}$, combinatorial constraint $m$, fairness constraint $\mathbf{k}$, time horizon $T$ and queue weight $\eta$.
  \STATE \textbf{Initialization}: 
  $Q_i(0) = 0$, $\alpha_i(0) = \beta_i(0) = 1$, $\forall i \in \mathcal{N}$;
  \FOR{$t=0, \ldots, T-1$}
  \STATE Observe the available arm set $Z(t)$;
  \STATE For each arm $i\in Z(t)$, draw a sample $\theta_i \sim \text{Beta}(\alpha_i(t), \beta_i(t))$;
  \STATE Pull arms ${A}(t)$ according to (\ref{eq:solution});
  \STATE Observe rewards $X_i, \forall i \in {A}(t)$;
  \STATE Update $Q_i(t+1)$ based on (\ref{eq:updateQ});
  \FORALL{$i \in {A}(t)$}
  \STATE Update $\alpha_{i}(t)$ and $\beta_{i}(t)$ based on (\ref{eq:updatebeta});
  \ENDFOR
  \ENDFOR
\end{algorithmic}
\end{algorithm}

To address the first challenge, we adopt the Thompson sampling technique with beta priors and Bernoulli likelihoods to achieve the tradeoff between the exploration and exploitation. The main idea is to assume a beta prior distribution with the shape parameters $\alpha_{i}(t)$ and $\beta_{i}(t)$ (i.e., ${\rm Beta}(\alpha_{i}(t),\beta_{i}(t))$) on the mean reward of each arm $u_i$. 
Initially, we let $\alpha_{i}(0)=\beta_{i}(0)=1$, since we have no knowledge about each $u_i$ and ${\rm Beta}(1,1)$ is a uniform distribution in $[0,1]$. Then, after observing the available arms $Z(t)$, we draw a sample $\theta_i(t)$ from ${\rm Beta}(\alpha_{i}(t),\beta_{i}(t))$ as an estimate for $u_i, \forall i \in Z(t)$, and pull arms $A(t)$ according to (\ref{eq:solution}) as discussed later. The arms in $A(t)$ return rewards $X_i(t), \forall i \in A(t)$, which are used to update the beta distributions based on Bayes rules and Bernoulli likelihood for all arms in $A(t)$:
\begin{equation}\label{eq:updatebeta}
\begin{aligned}
  &\alpha_i(t+1) = \alpha_i(t) + X_i(t),\\
  &\beta_i(t+1) = \beta_i(t) + 1 - X_i(t).
\end{aligned}
\end{equation}
After a number of rounds, we are able to see that the mean of the posterior beta distributions will converge to the true mean of the reward distributions.

The virtual queue technique~\cite{li2019combinatorial,neely2010stochastic} can be used to ensure that the fairness constraints are satisfied. The high-level idea  behind the design is to establish a time-varying queue $Q_i(t)$ to record the number of times that arm $i$ has failed to meet the fairness. Initially, we set $Q_i(0) = 0$ for all $i \in \mathcal{N}$. For the ease of presentation, let $d_i(t) := \mathbf{1}[i \in {A}(t)]$ be a binary random variable indicating that whether arm $i$ is pulled or not in round $t$. Then for each arm $i \in \mathcal{N}$, we can use the following way to maintain the queue:
\begin{equation}\label{eq:updateQ}
  Q_i(t) = \max \left\{t \cdot k_i - \sum_{\tau=0}^{t-1} d_i(t), \ 0 \right\}.
\end{equation}
Intuitively, the length of the virtual queue for arm $i$ increases $k_i$ if the arm is not pulled in  round $t$.  Therefore, arms with longer queues are more unfair and will be given a higher priority to be pulled in future rounds. Note that our queue evolution is slightly different from \cite{li2019combinatorial} to avoid the rounding error accumulation issue.

To further balance the fairness and the reward, we introduce another parameter $\eta$ as a tradeoff between the reward and the virtual queue lengths. Then, in each round $t$, the learning agent pulls arms $A(t)$ as follows:


\begin{equation}\label{eq:solution}
    A(t) \in \underset{A \subseteq {Z}(t), |A| \leq m}{\operatorname{argmax}} \sum_{i \in A}\left(\frac{1}{\eta} Q_{i}(t) +  w_{i} \theta_i(t)\right).
\end{equation}
Note that different from LFG, we weigh $Q_i(t)$ with $\frac{1}{\eta}$ in (\ref{eq:solution}) rather than weighing $\theta_i(t)$ with $\eta$. The advantage is that we can simply let $\eta \rightarrow \infty$ to neglect virtual queues, so the algorithm can be adapted to CSMAB easily. The whole process of the TSCSF-B algorithm is shown in Alg.~\ref{alg:CombTS1}.

\section{Results and Proofs}\label{sec:results}
\subsection{Fairness Satisfaction}
\begin{theorem}\label{thm:feasbility}
For any fixed and finite $\eta>0$, when $T$ is long enough and the fairness constraint vector $\mathbf{k}$ is feasible, the proposed TSCSF-B algorithm satisfies the long-term fairness constraints defined in (\ref{eq:fraction}).
\end{theorem}
\begin{proof}[Proof Sketch]
The main idea to prove Theorem \ref{thm:feasbility} is to prove the virtual queue for each arm is stable when $\mathbf{k}$ is feasible and $T$ is long enough for any fixed and finite $\eta>0$.  The proof is based on Lyapunov-drift analysis~\cite{neely2010stochastic}. Since it is not our main contribution and it follows similar lines to the proof of Theorem 1 in \cite{li2019combinatorial},  we omit the proof. Interested readers are referred to \cite{li2019combinatorial}.
\end{proof}
\begin{remark}
The long-term fairness constraints does not require arms to be pulled for a certain number of times in each round but by the end of the time horizon. Theorem \ref{thm:feasbility} states that the fairness constraints can always be satisfied by TSCSF-B as long as $\eta$ is finite and $T$ is long enough. A higher $\eta$ may require a longer time for the fairness constraints to be satisfied~(see Sec.~\ref{sec:evaluation}).
\end{remark}

\subsection{Regret Bounds}
\begin{theorem}\label{thm:regret}
For any fixed $T>1, \eta>0, w_{\max} > 0$ and $m \in (0,N]$, the time-averaged regret of TSCSF-B is upper bounded by
\begin{equation*}
   \frac{N}{2\eta} + \frac{4w_{\max}\sqrt{mNT\ln T} + 2.51 w_{\max} N }{T}.
\end{equation*}
\end{theorem}
\begin{proof}[Proof Sketch]\label{skpf:regret1}
We only provide a sketch of proof here, and the detailed proof can be found in Appendix. The optimal policy for CSMAB-F is a randomized algorithm defined in Sec.~\ref{sec:problem}, while the optimal policies for classic MAB problems are deterministic. We follow the basic idea in \cite{li2019combinatorial} to convert the regret bound between the randomized optimal policy and TSCSF-B~(i.e., regret) by the regret bound between a deterministic oracle and TSCSF-B. The deterministic oracle also knows the mean reward for each arm, and can achieve more rewards than the optimal policy by sacrificing fairness constraints a bit. Denote the arms pulled by the oracle in round $t$ as $A^{\prime}(t)$, which is defined by
\begin{equation*}
    A^{\prime}(t) \in \underset{A \subseteq Z(t), |A| \leq m}{\operatorname{argmax}} \sum_{i \in A}\left(\frac{1}{\eta} Q_{i}(t) +  w_{i} u_{i}(t)\right).
\end{equation*}
Then, we can prove that the time-averaged regret defined in (\ref{eq:regret}) is bounded by
\begin{equation}\label{eq:astar}
\resizebox{.9\hsize}{!}{$
\frac{N}{2 \eta}+\frac{1}{T} \left(\underbrace{\sum_{t=0}^{T-1} \mathbb{E}\left[\sum_{i \in A(t)} w_{i}\left(\theta_{i}(t)-u_{i}\right)\right] + \sum_{t=0}^{T-1}\mathbb{E}\left[\sum_{i \in A^{\prime}(t)} w_{i}\left(u_{i}-\theta_{i}(t)\right)\right]}_{C}\right)$,}
\end{equation}
where the first term $\frac{N}{2\eta}$ is due to the queuing system, and the second part $C$ is due to the exploration and exploitation.

Next, we define two events and their complementary events for each arm $i$ to decompose $C$. Let $\gamma_i(t) := \sqrt{\frac{2\ln T}{h_i(t)}}$, where $h_i(t)$ is the number of times that arm $i$ has been pulled at the beginning of round $t$. Then for each arm $i\in \mathcal{N}$, the two events $\mathcal{J}_i(t)$ and $\mathcal{K}_i(t)$ are defined as follows:
\begin{equation*}
\begin{aligned}
   &\mathcal{J}_i(t) := \{\theta_i(t)  - u_i > 2\gamma_i(t) \},\\
   &\mathcal{K}_i(t) := \{u_i - \theta_i(t)> 2\gamma_i(t) \},
\end{aligned}
\end{equation*}
and let $\overline{\mathcal{J}_i(t)}$ and $\overline{\mathcal{K}_i(t)}$ be the complementary events for  $\mathcal{J}_i(t)$ and $\mathcal{K}_i(t)$, respectively.
Notice that both $\mathcal{J}_i(t)$ and $\mathcal{K}_i(t)$ are low-probability events after a number of rounds.

With the events defined above, we can decompose $C$ as
\begin{equation*}
\resizebox{1.0\hsize}{!}{$
\begin{aligned}
     &\underbrace{\sum_{t=0}^{T-1} \mathbb{E}\left[\sum_{i\in A(t)} w_{i}\left(\theta_{i}(t)-u_{i}\right) \mathbf{1}[\mathcal{J}_i(t)] \right]}_{B_1}
     + \underbrace{\sum_{t=0}^{T-1} \mathbb{E}\left[\sum_{i\in A(t)} w_{i}\left(\theta_{i}(t)-u_{i}\right) \mathbf{1}[\overline{\mathcal{J}_i(t)}] \right]}_{B_2}\\
     + &\underbrace{\sum_{t=0}^{T-1} \mathbb{E}\left[\sum_{i\in A^\prime(t)} w_{i}\left(\theta_{i}(t)-u_{i}\right) \mathbf{1}\left[\mathcal{K}_i(t)\right]\right]}_{B_3}
     + \underbrace{\sum_{t=0}^{T-1} \mathbb{E}\left[\sum_{i\in A^\prime (t)} w_{i}\left(\theta_{i}(t)-u_{i}\right)  \mathbf{1}\left[\overline{\mathcal{K}_i(t)}\right] \right]}_{B_4}.
\end{aligned}$}
\end{equation*}

Using the relationship between the summation and integration, we can bound $B_2$ and $B_4$ by $4w_{\max} \sqrt{mNT\ln T} + w_{\max}N$.

Bounding $B_1$ and $B_3$ is the main theoretical contribution of our work and is not trivial. Since $\mathcal{J}_i(t)$ and $\mathcal{K}_i(t)$ are low-probability events, the total times they can happen are a constant value on expectation. Therefore, to bound $B_1$ and $B_3$,  the basic idea is to obtain the bounds for $\Pr(\mathcal{J}_i(t))$ and $\Pr(\mathcal{K}_i(t))$. Currently, there is no existing work giving the bounds for $\Pr(\mathcal{J}_i(t))$ and $\Pr(\mathcal{K}_i(t))$, and we prove that $\Pr(\mathcal{J}_i(t))$ and $\Pr(\mathcal{K}_i(t))$ are bounded by $\left(\frac{1}{T^2}+\frac{1}{T^{32}}\right)$ and $\left(\frac{1}{T^8} + \frac{1}{T^{32}} \right)$, respectively. Then, it is straightforward to bound $B_1$ and $B_3$ by $2.51w_{\max}N$.

\end{proof}
\begin{remark}
Comparing with the time-averaged regret bound for LFG~\cite{li2019combinatorial},  we have the same first term $\frac{N}{2\eta}$, as we adopt the virtual queue system to satisfy the fairness constraints. On the other hand, the second part of our regret bound, which is also the first problem-independent regret bound for CSMAB problems, has lower coefficients than that of LFG. Specifically, the coefficient for the time-dependent term (i.e., $\sqrt{mNT\ln T}$) is $4$ in our bound, smaller than $2\sqrt{6}$ in that of LFG, and the time-independent term (i.e., $w_{\max}N$) has a coefficient $2.51$ in our bound, which is also less than $5.11$ in the bound of LFG.

If we let $\eta \rightarrow \infty$, the algorithm only focuses on CSMAB problems, and the bound boils down to the first problem-independent bound of TS-based algorithms for CSMAB problems, which matches the lower bound proposed in \cite{kveton2015tight}.
\end{remark}
\begin{corollary}\label{cr:1}
For any fixed $m \in (0,N]$ and $\eta \geq \sqrt{\frac{NT}{m\ln T}}$, when $T \geq N$, the time-averaged regret of TSCSF-B is upper bounded by $\widetilde{O}\left(\frac{\sqrt{mNT}}{T}\right).$
\end{corollary}
\begin{remark}
The reason we let $\eta \geq \sqrt{\frac{NT}{m\ln T}}$ for a given $T$ is to control the first term to have a consistent or lower order than the second term. However, in practice, we need to tune $\eta$ according to $T$ such that both fairness constraints and high rewards can be achieved.
\end{remark}




\section{Evaluations and Applications}\label{sec:evaluation}

\subsection{Numerical Experiments}
In this section, we compare the TSCSF-B algorithm with the LFG algorithm~\cite{li2019combinatorial} in two settings. The first setting is identical to the setting in~\cite{li2019combinatorial},  where $N=3$, $m=2$, and $w_i = 1, \forall i \in \mathcal{N}$. The mean reward vector for the three arms is $(0.4, 0.5, 0.7)$. The availability of the three arms is $(0.9,0.8,0.7)$, and the fairness constraints for the three arms are $(0.5,0.6,0.4)$. To see the impact of $\eta$ on the time-averaged regret and fairness constraints, we compare the algorithms under $\eta=1, 10, 1000$ and $\infty$ in a time horizon $T = 2 \times 10^4$, where $\eta \rightarrow \infty$ indicates that both algorithms do not consider the long-term fairness constraints.

Further, we test the algorithms in a more complicated setting where $N=6$, $m=3$, and $w_i = 1, \forall i \in \mathcal{N}$. The mean reward vector for the six arms is $(0.52,0.51,0.49,0.48,0.7,0.8)$. The availability of the six arms is $(0.7,0.6,0.7,0.8,0.7,0.6)$, and the fairness constraints for the six arms are $(0.4,0.45,0.3,0.45,0.3,0.4)$.
This setting is challenging because higher fairness constraints are given to the arms with less mean rewards and the arms with lower availability (i.e., arms $2$ and $4$).
According to Corollary \ref{cr:1}, we set $\eta=\sqrt{\frac{NT}{m\ln T}} = 63.55$
and $\infty$, and $T = 2 \times 10^4$.
Note that the following results are the average of $100$ independent experiments. We omit the plotting of confidence interval and deviations because they are too small to be seen from the figures and are also omitted in most bandit papers.

\subsubsection{Time-Averaged Regret}
The time-averaged regret results under the first setting and the second setting are shown in Fig.~\ref{fig:3regret} and Fig.~\ref{fig:setting2}, respectively.
\begin{figure}[h]
\subfloat [$\eta = 1$] {\label{fig:3regret:a}
\epsfig{file=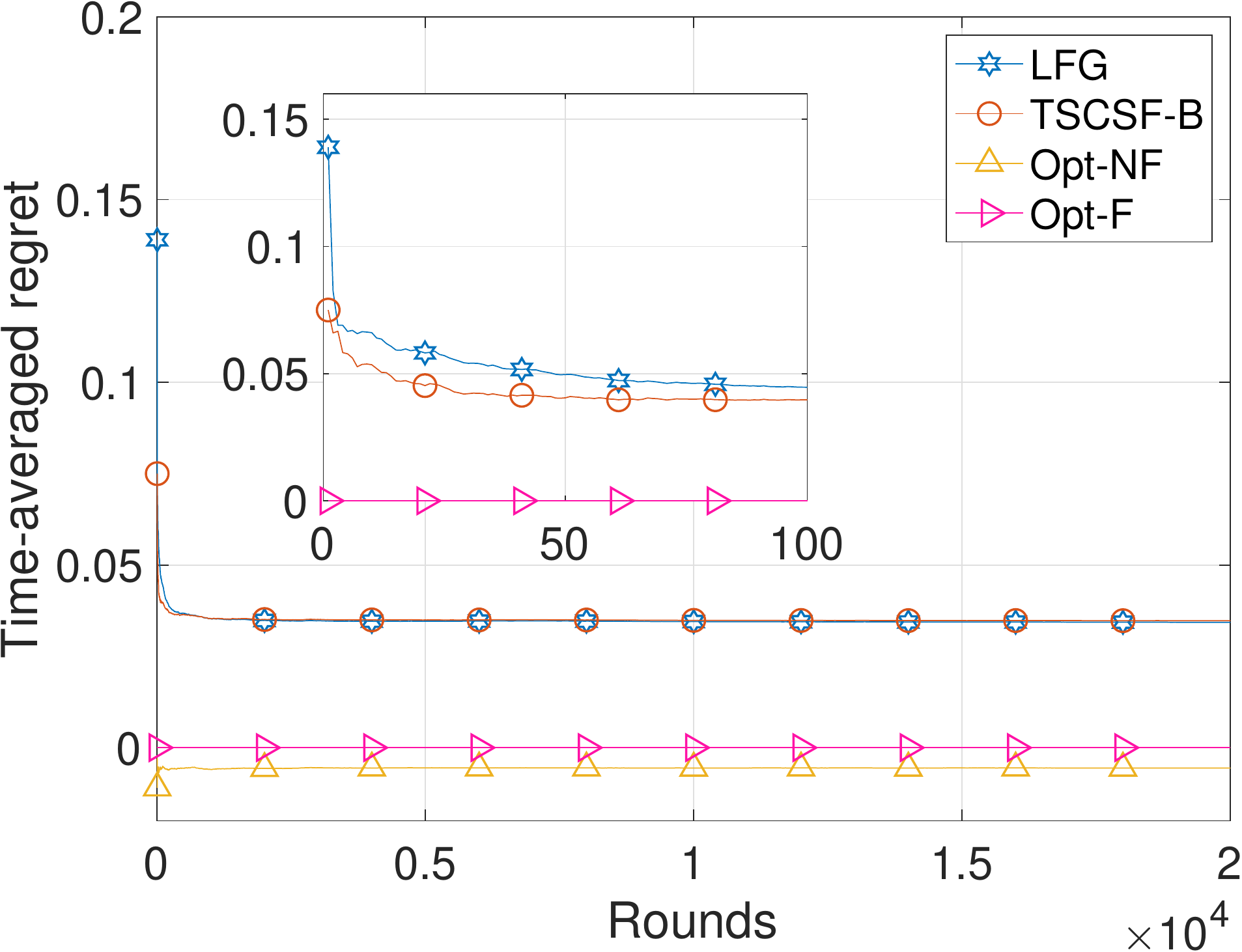, width = 0.48\columnwidth}
}
\subfloat [$\eta = 10$] {\label{fig:3regret:b}
\epsfig{file = 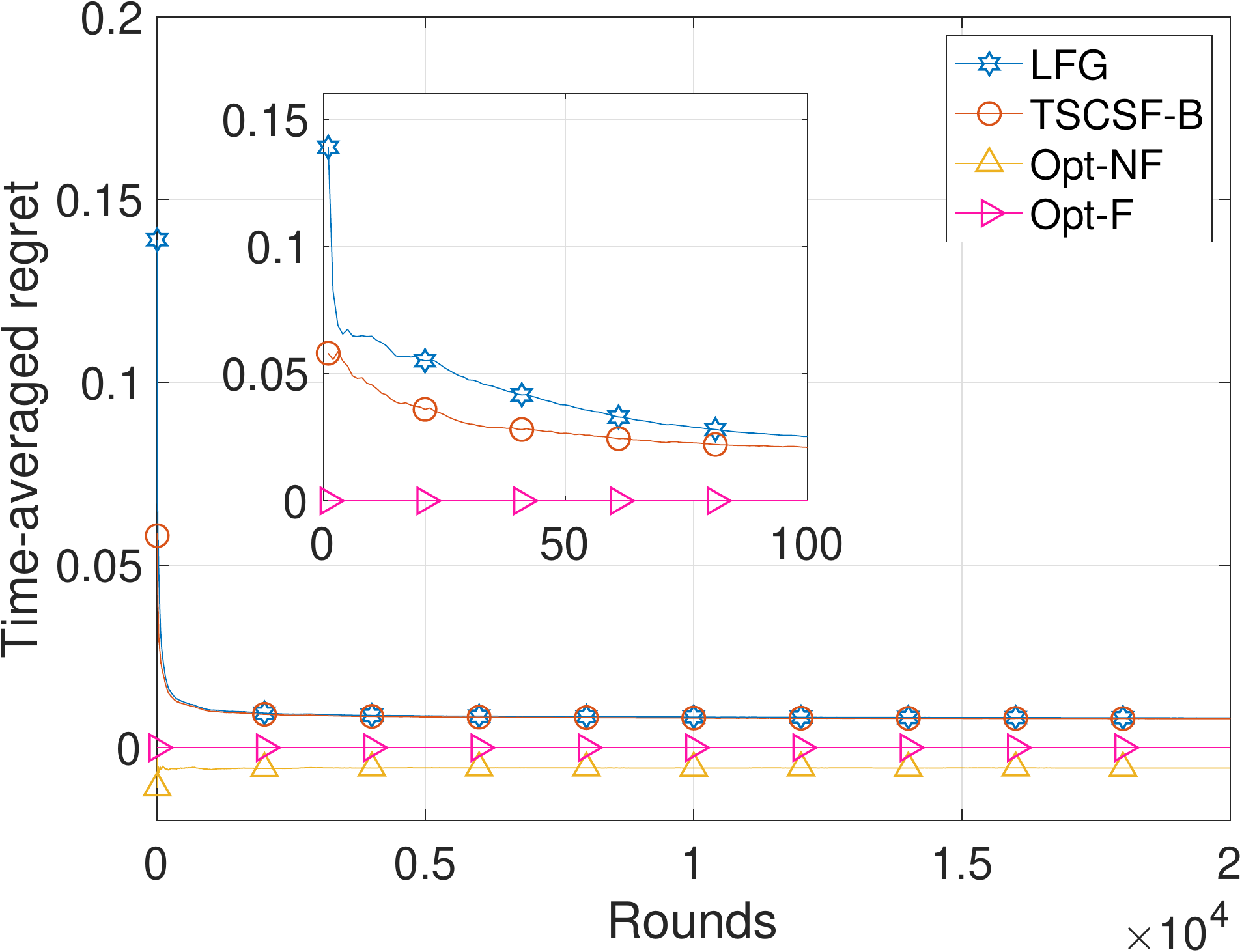, width = 0.48\columnwidth}
}

\subfloat [$\eta = 1000$] {\label{fig:3regret:d}
\epsfig{file = 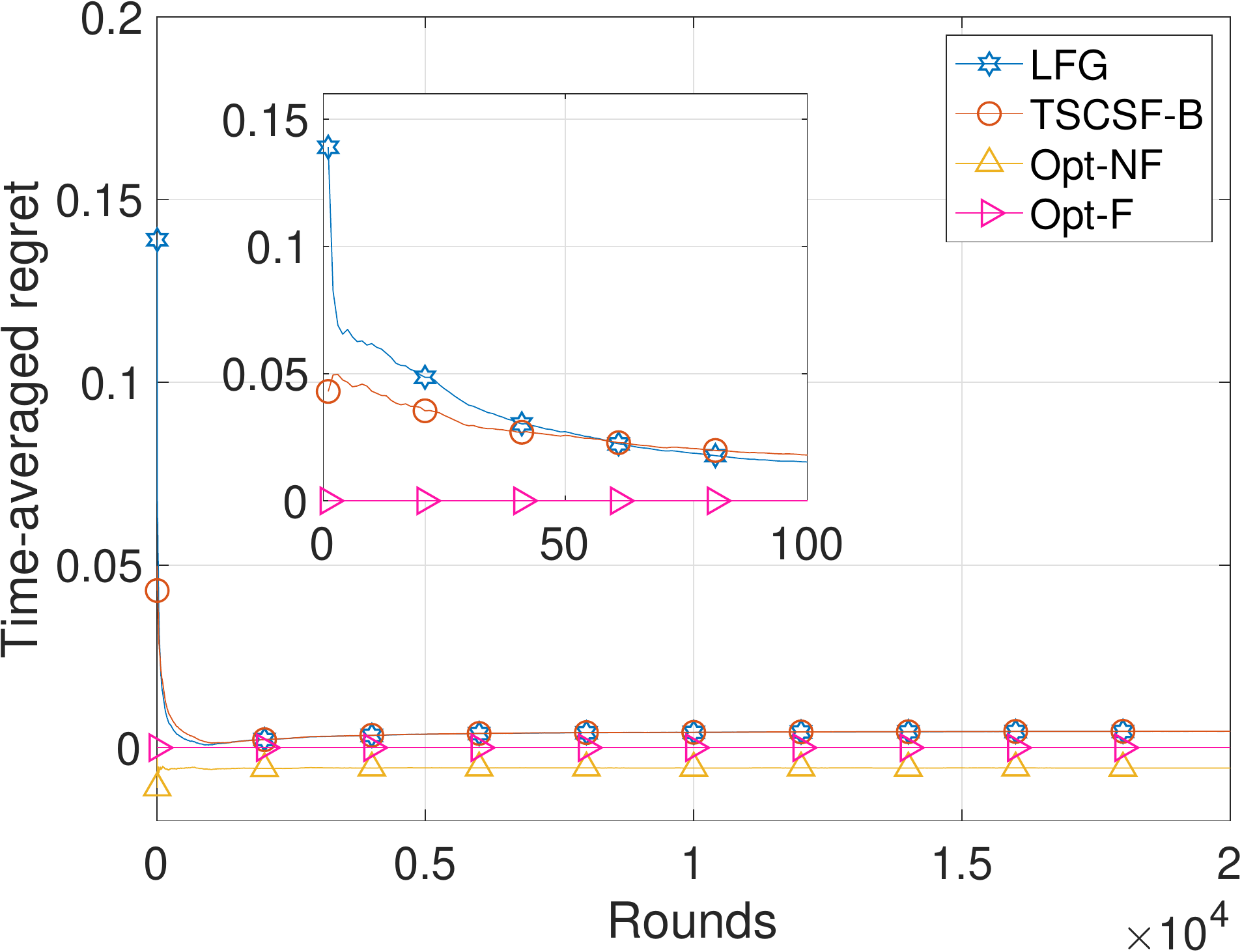, width = 0.48\columnwidth}
}
\subfloat [$\eta \rightarrow \infty $] {\label{fig:3regret:e}
\epsfig{file = 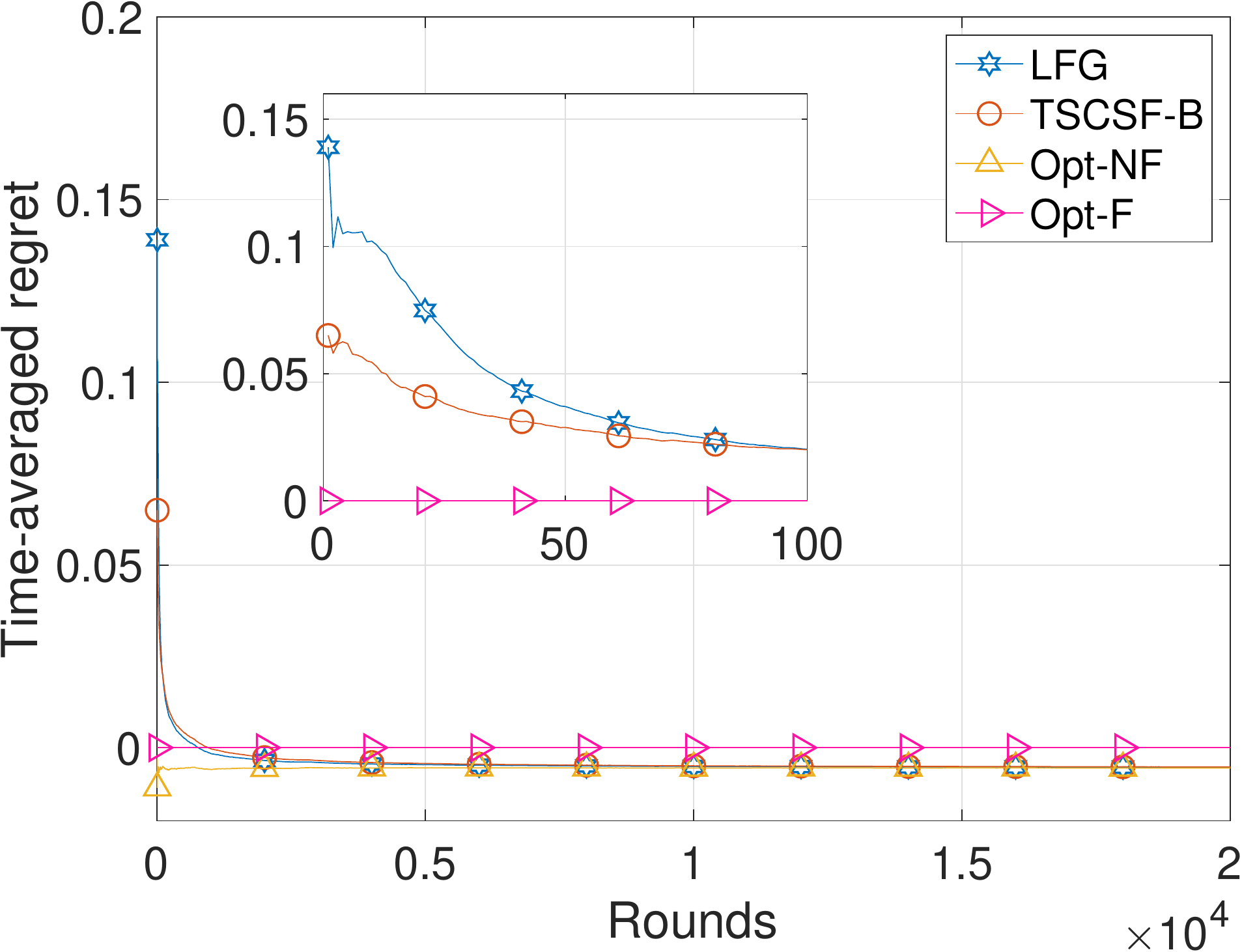, width = 0.48\columnwidth}
}
\caption{Time-averaged regret for the first setting.}\label{fig:3regret}
\vspace{-13pt}
\end{figure}

\begin{figure}
\subfloat [$\eta = \sqrt{\frac{NT}{m\ln T}}$] {\label{fig:setting2:a}
\epsfig{file=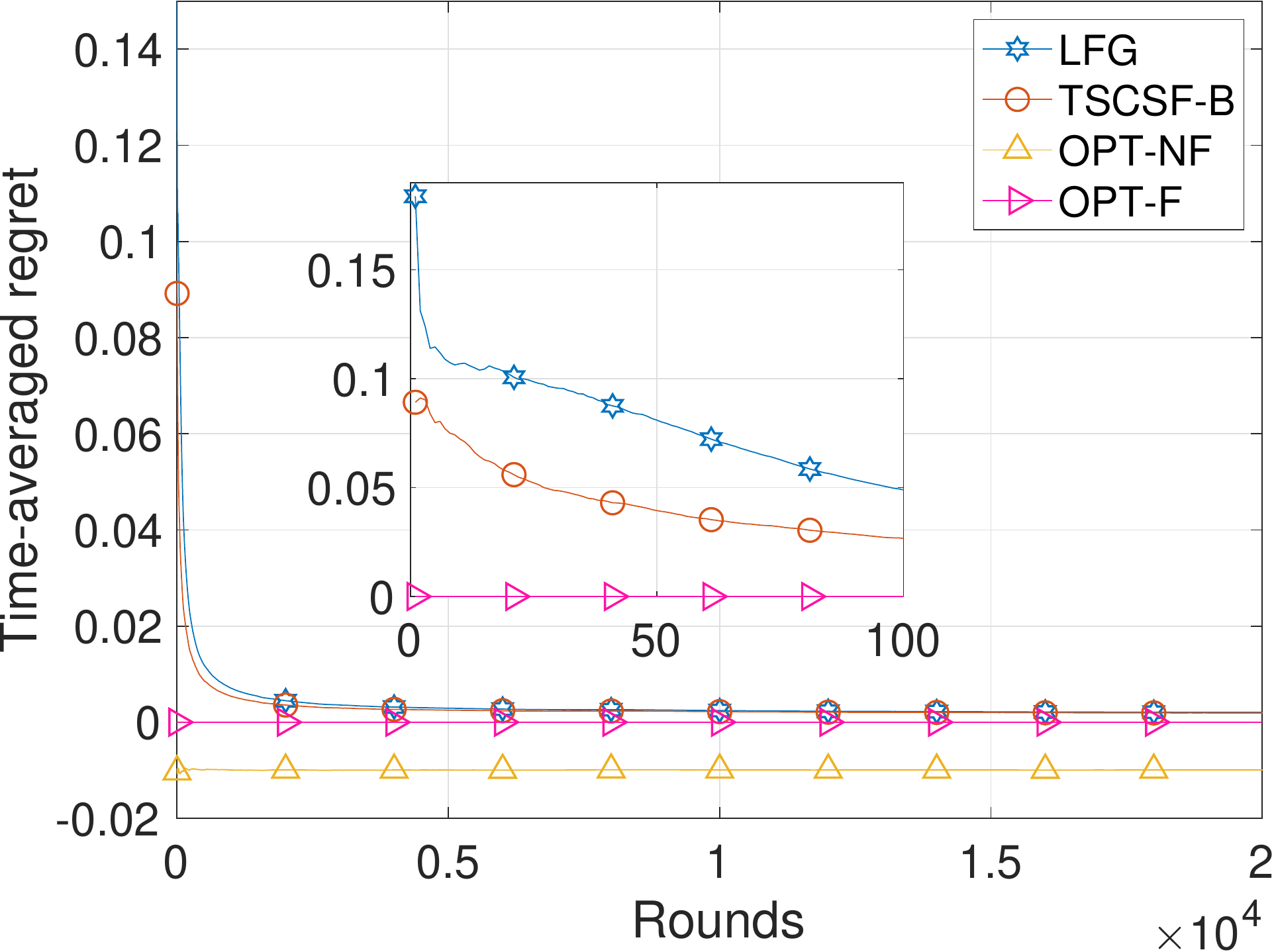, width = 0.48\columnwidth}
}
\subfloat [$\eta \rightarrow \infty $] {\label{fig:setting2:c}
\epsfig{file= 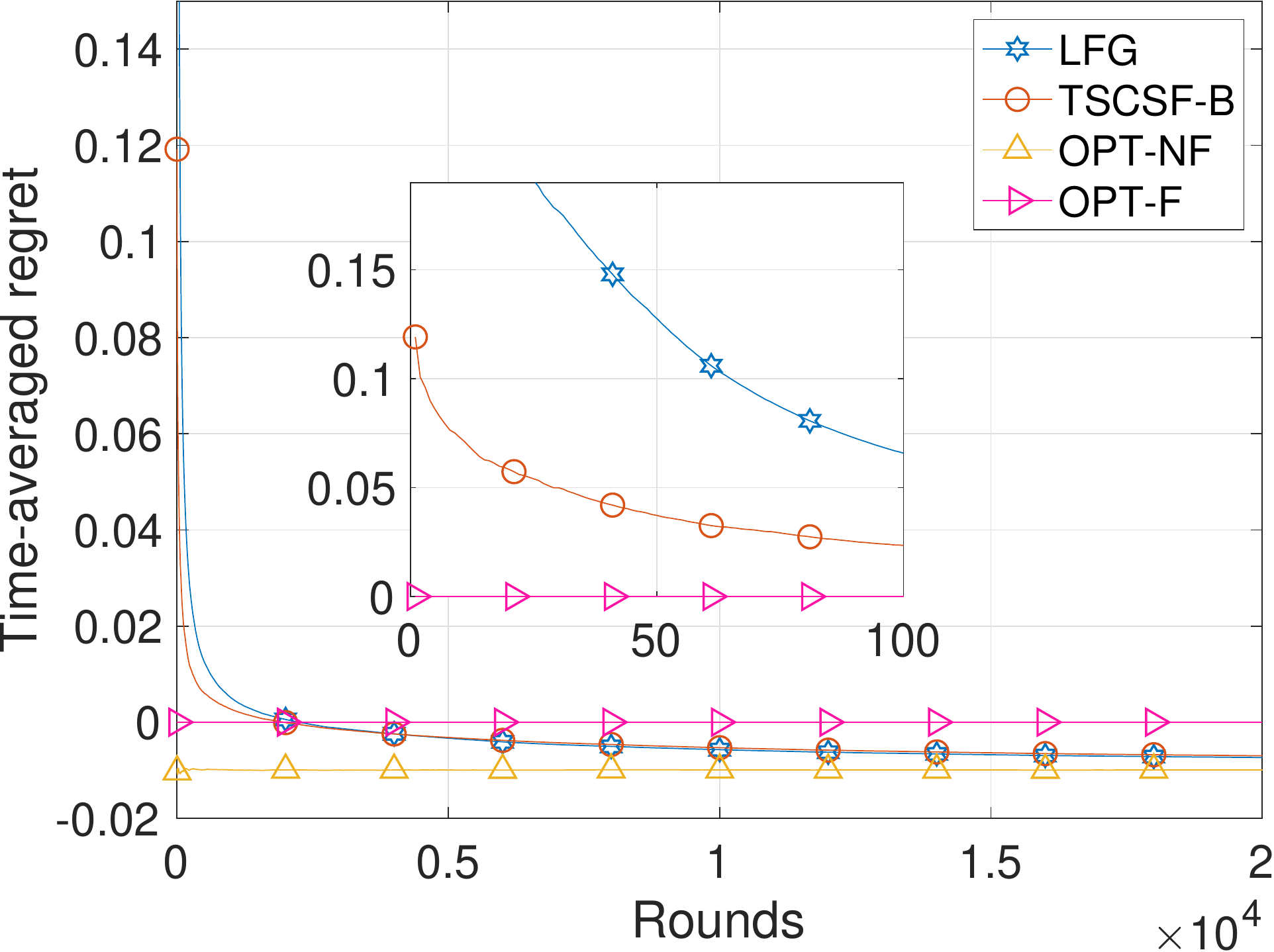, width = 0.48\columnwidth}
}

\caption{Time-averaged regret for the second setting.}\label{fig:setting2}
\vspace{-13pt}
\end{figure}
In each subplot, the $x$-axis represents the rounds and the $y$-axis is the time-averaged regret. A small figure inside each subplot zooms in the first $100$ rounds. We also plot the \emph{OPT with considering fairness~(Opt-F)} (i.e., the optimal solution to CSMAB-F), and \emph{OPT without considering fairness~(Opt-NF)} (i.e., the optimal solution to CSMAB). The time-averaged regret of Opt-NF is always below Opt-F, since Opt-NF does not need to satisfy the fairness constraints and can always achieve the highest rewards. By definition, the regret of Opt-F is always $0$.

We can see that the proposed TSCSF-B algorithm has a better performance than the LFG algorithm, since it converges faster, and achieves a lower regret, as shown in Fig.~\ref{fig:3regret} and Fig.~\ref{fig:setting2}. It is noteworthy that the gap between TSCSF-B and LFG is larger in Fig.~\ref{fig:setting2}, which indicates that TSCSF-B performs better than LFG in more complicated scenarios.

In terms of $\eta$, the algorithms with a higher $\eta$ can achieve a lower time-averaged regret. For example, in the first setting, the lowest regrets achieved by the two considered algorithms are around $0.03$ when $\eta = 1$,  but they are much closer to Opt-F when $\eta = 10$. However, when we continue to increase $\eta$ to $1000$ (see Fig.~\ref{fig:3regret:d}), the considered algorithms achieve a negative time-averaged regret around $0.2 \times 10^4$ rounds, but recover to the positive value afterwards. This is due to the fact that with a high $\eta$ the algorithms preferentially pull arms with the highest mean rewards, but the queues still ensure the fairness can be achieved in future rounds.
When $\eta \rightarrow \infty$ (see Fig.~\ref{fig:3regret:e} and Fig.~\ref{fig:setting2:c}), the fairness constraints are totally ignored and the regrets of the considered algorithms converge to Opt-NF. Therefore, $\eta$ significantly determines whether the algorithms can satisfy and how quickly they satisfy the fairness constraints.

\subsubsection{Fairness Constraints}

In the first setting, we show in Fig.~\ref{fig:fairness1} the final satisfaction of fairness constraints for all  arms under $\eta = 1000$.  $\eta = 1000$ is an interesting setting where the fairness constraints are not satisfied in the first few rounds as aforementioned. We want to point out in the first setting, the fairness constraint for arm $1$ is relatively difficult to be satisfied, since arm $1$ has the lowest mean reward but has a relative high fairness constraint.  However,  we can see that the fairness constraints for all arms are satisfied finally, which means both TSCSF-B and LFG are able to ensure the fairness constraints in this simple setting.

In the second setting with $\eta = \sqrt{\frac{NT}{m\ln T}} = 63.55$, the fairness constraints for arms $2$ and $4$ are difficult to satisfy, as both arms have high fairness constraints but low availability or low mean reward. However, both TSCSF-B and LFG manage to satisfy the fairness constraints for all the $6$ arms, as shown in  Fig.~\ref{fig:fairness2}.

\begin{figure}[h]
    \centering
    \includegraphics[width = 0.8\columnwidth]{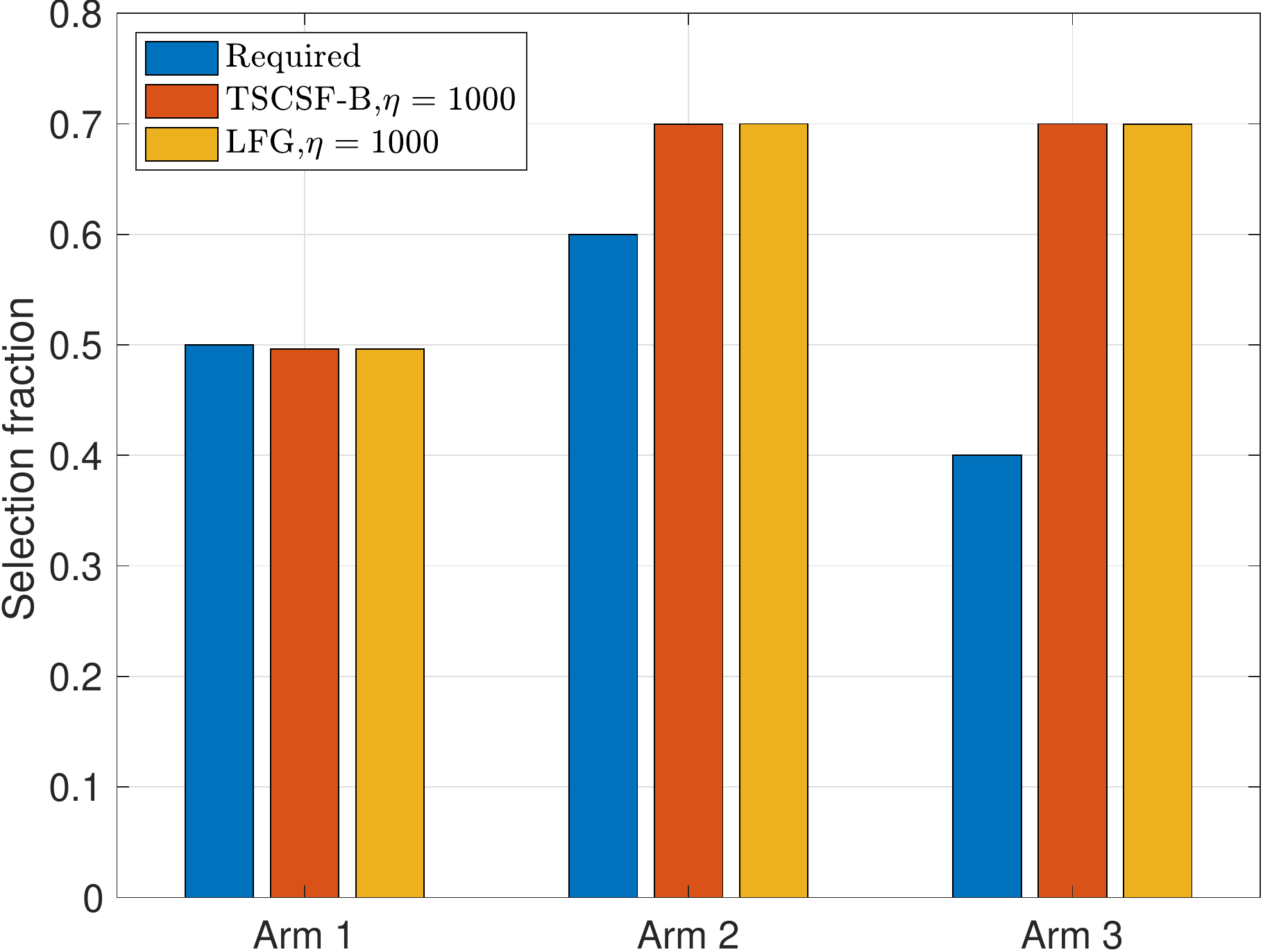}
    \caption{Satisfaction of fairness constraints in the first setting.}
    \label{fig:fairness1}
\end{figure}

\begin{figure}[h]
    \centering
    \includegraphics[width = 0.8\columnwidth]{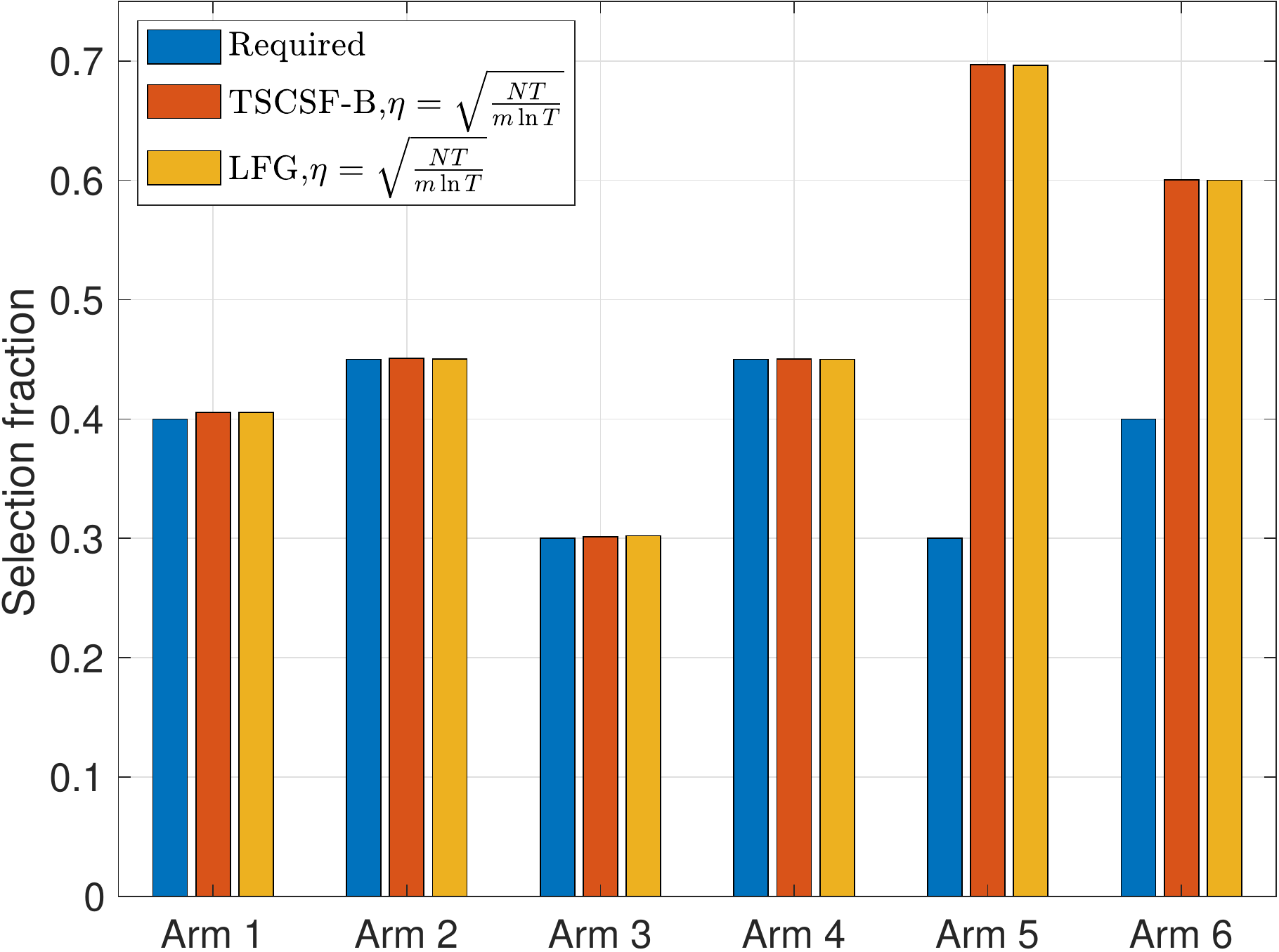}
    \caption{Satisfaction of fairness constraints in the second setting.}
    \label{fig:fairness2}
\end{figure}


\subsection{Tightness of the Upper bounds}
Finally, we show the tightness of our bounds in the second setting, as plotted in Fig.~\ref{fig:tightness}. The $x$-axis represents the change of the time horizon $T$, and the $y$-axis is the logarithmic time-averaged regret in the base of $e$.

We can see that, the upper bound of TSCSF-B is always below that of LFG. However, there is a big gap between the TSCSF-B upper bound and the actual time-averaged regret in the second setting. This is reasonable, since the upper bound is problem-independent, but it is still of interest to find tighter bound for CSMAB-F problems.


\begin{figure}
    \centering
    \includegraphics[width=\columnwidth]{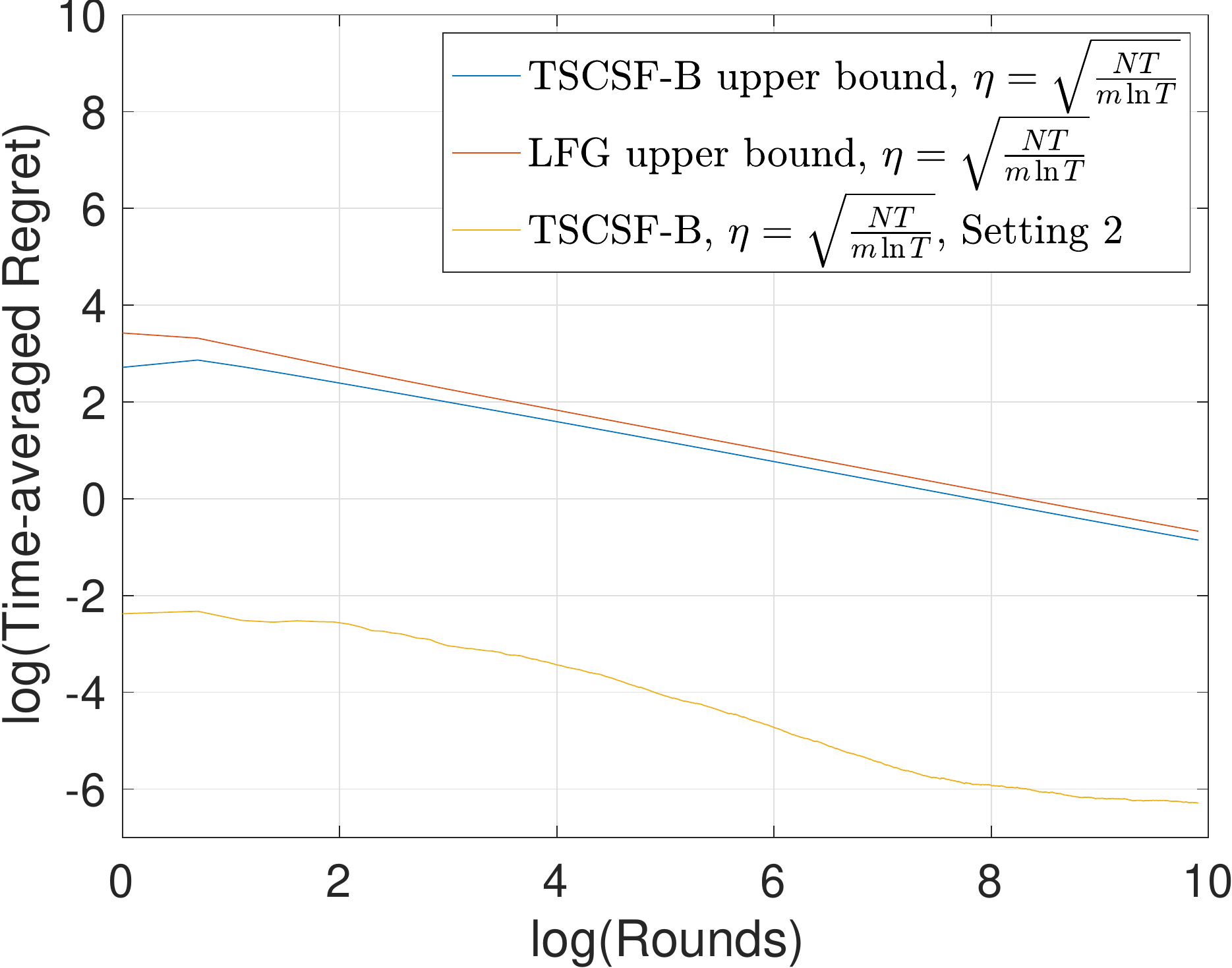}
    \caption{Tightness of the upper bounds for TSCSF-B}
    \label{fig:tightness}
\end{figure}

\subsection{High-rating movie recommendation System}
In this part, we consider a high-rating movie recommendation system. The objective of the system is to recommend high-rating movies to users, but the ratings for the considered movies are unknown in advance. Thus, the system needs to learn the ratings of the movies while simultaneously recommending the high-rating ones to its users. Specifically, when each user comes, the movies that are relevant to the user's preference are available to be recommended. Then, the system recommends the user with a subset of the available movies subject.
After consuming the recommended movies, the user gives feedback to the system, which can be used to update the ratings of the movies to better serve the upcoming users. In order to acquire accurate ratings or to ensure the diversity of recommended movies, each movie should be recommended at least a number of times.

The above high-rating movie recommendation problem can be modeled as a CSMAB-F problem under three assumptions.
First, we take serving one user as a round by assuming the next user always comes after the current user finishes rating. This assumption can be easily relaxed by adopting \emph{the delayed feedback framework} with an additive penalty to the regret~\cite{joulani2013online}.
Second, the availability set of movies is stochastically generated according to an unknown distribution.
Last, given a movie, the ratings are i.i.d. over users with respect to an unknown distribution. The second and third assumptions are feasible, as it has been discovered that the user preference and ratings towards movies have a strong relationship to the Zipf distribution~\cite{Cha07I,gill07youtube}.

\subsubsection{Setup}
We implement TSCSF-B and LFG on \emph{MovieLens 20M Dataset}~\cite{harper2016movielens}, which includes $20$ million ratings to $27,000$ movies by $138,000$ users. This dataset contains both users’ movie ratings between $1$ and $5$ and genre categories for each movie. In order to compare the proposed TSCSF-B algorithm to the LFG algorithm, we select $N=5$ movies with different genres as the ground set of arms $\mathcal{N}$, which are Toy Story (1995) in the genre of Adventure, Braveheart (1995) in Action, Pulp Fiction (1994) in Comedy, Godfather, The (1972) in Crime,  and Alien (1979) in Horror.

Then, we count the total number of ratings on the selected $5$ genres and calculate occurrence of each selected genre among the $5$ genres as the availability of the corresponding selected movie. We note that the availability of the selected movies is only used by the OPT-F algorithm and is not used to determine the available set of movies in each round. During the simulation, when each user comes, the available set of movies is determined by whether the user has rated or not these movies in the dataset.

The ratings are scaled into $[0,1]$ to satisfy as the rewards. We choose $28,356$ users who have rated at least one of the selected $5$ movies as the number of rounds~(one round one user according to the first assumption) for the algorithms, and take their ratings as the rewards to the recommended movies.
When each user comes, the system will select no more than $m=2$ movies for recommendation and each movie shares the same weight, i.e., $w_i = 1, \forall i \in \mathcal{N}$, and the same fairness constraints. The fairness constraints are set as $\mathbf{k} = [0.3,0.3,0.3,0.3,0.3]$ such that (\ref{eq:question}) has a feasible solution.

We adopt such an implementation, including the determination of the available movie set, the same movie weights, and the same fairness constraints, to ensure that our simulation brings noise as little as possible to the MovieLens dataset.

\subsubsection{Results}
We first show whether the considered algorithms are able to achieve accurate ratings. The final ratings of selected movies by TSCSF-B and LFG under $\eta = \sqrt{\frac{NT}{m\ln T}}$ and  $\infty$ are shown in Fig.~\ref{fig:Movieratings}. The reason why we set $\eta = O(\sqrt{\frac{NT}{m\ln T}})$ is due to Corollary~\ref{cr:1}.
We can observe that the performance of TSCSF-B is better than that of LFG, since the ratings of TSCSF-B are much closer to the true average ratings, while the ratings acquired by UCB are higher than the true average ratings.

The final satisfaction for the fairness constraints of the selected movies is shown in Fig.~\ref{fig:fractionfinal}. Both TSCSF-B and LFG can satisfy the fairness constraints of the five movies under $\eta=\sqrt{\frac{NT}{m\ln T}}$.

\begin{figure}[h]
\subfloat [] {\label{fig:Movieratings}
\epsfig{file= 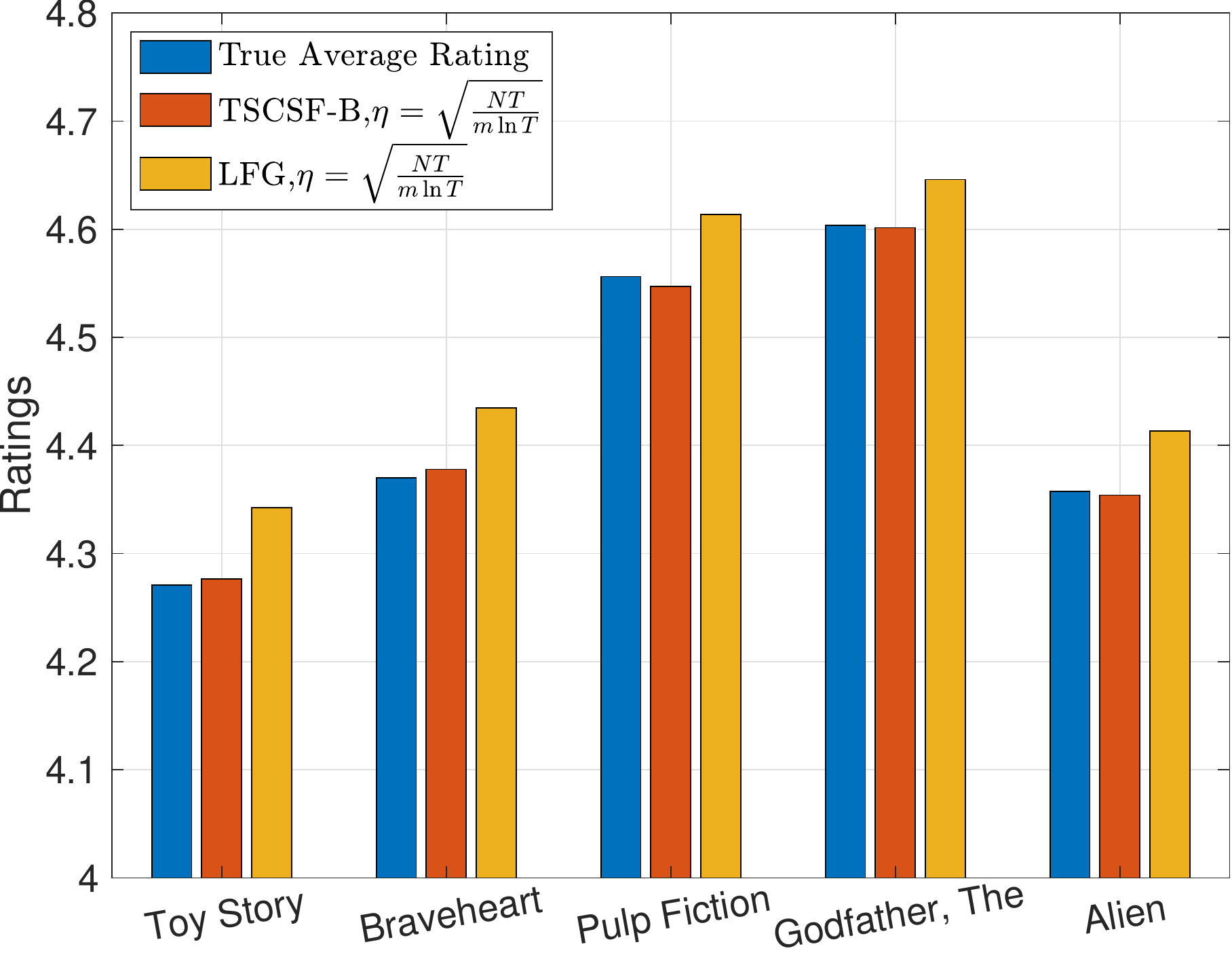, width = 0.48\columnwidth}
}
\subfloat [] {\label{fig:fractionfinal}
\epsfig{file= 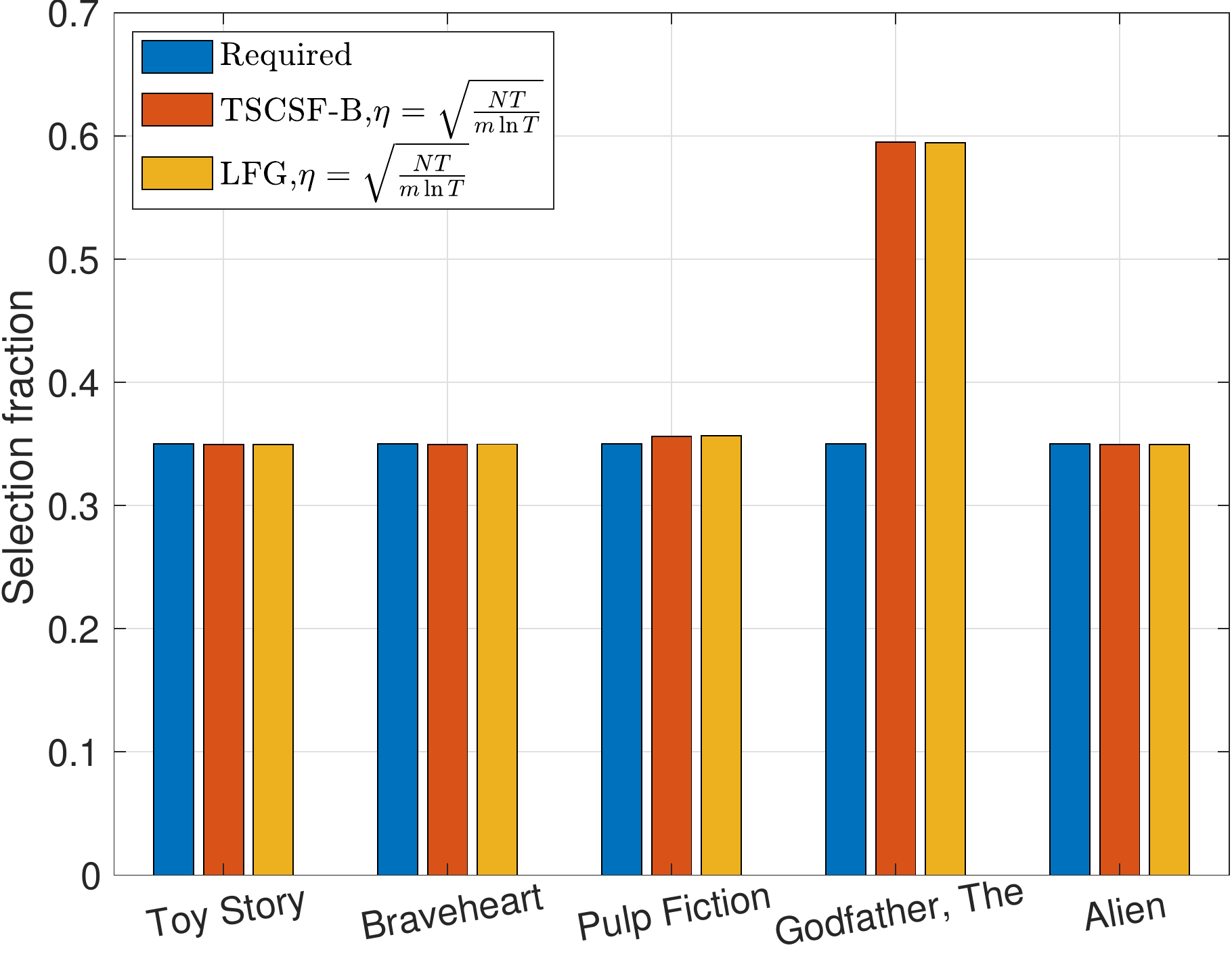, width = 0.48\columnwidth}
}
\caption{(a) The final ratings of selected movies. (b) The final satisfaction for the fairness constraints of selected movies.}
\vspace{-13pt}
\end{figure}

On the other hand, the time-averaged regret is shown in Fig.~\ref{sec:movieregret}. We can see that the time-averaged regret of TSCSF-B is below that of LFG, which indicates the proposed TSCSF-B algorithm converges much faster. Since we are unable to obtain the true distribution of the available movie set~(as discussed in \emph{Setup}), the rewards achieved by the OPT-F algorithm may not be the optimal one, which explains why the lines of both TSCSF-B and LFG are below that of OPT-F in Fig.~\ref{sec:movieregret}.

Generally,  TSCSF-B performs much better than LFG in this application, which achieves better final ratings and a quicker convergence speed.

\begin{figure}[h]
    \centering
    \includegraphics[width = 0.8\columnwidth]{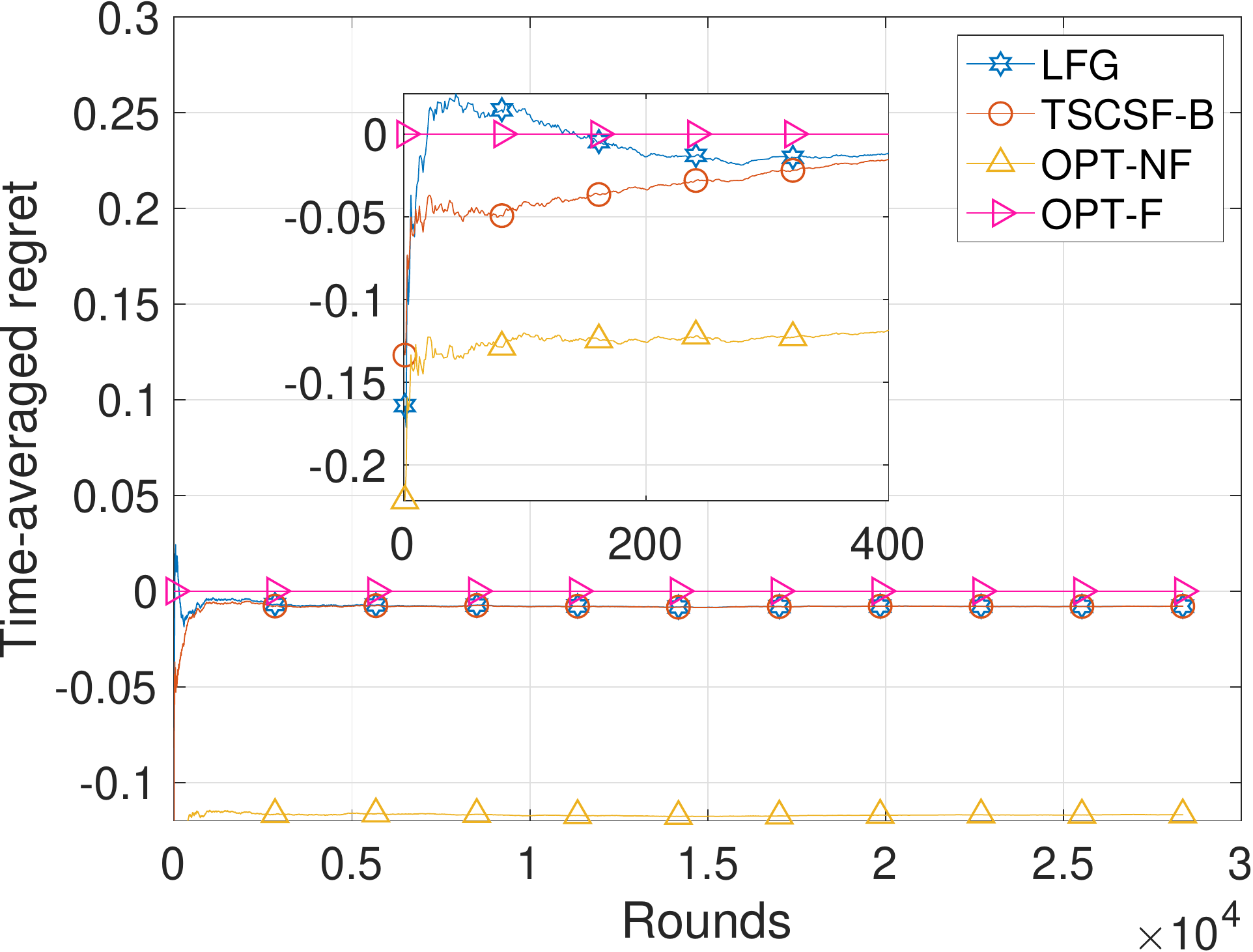}
    \caption{Time-averaged regret bounds for the high-rating movie recommendation system.}
    \label{sec:movieregret}
\end{figure}




\section{Conclusion}\label{sec:conclusion}
In this paper, we studied the stochastic combinatorial sleeping multi-armed bandit problem with fairness constraints, and designed the TSCSF-B algorithm with a provable problem-independent bound of $\widetilde{O}\left(\frac{\sqrt{mNT}}{T}\right)$ when $T \geq N$. Both the numerical experiments and real-world applications were conducted to verify the performance of the proposed algorithms.

As part of the future work, we would like to derive more rigorous relationship between $\eta$ and $T$ such that the algorithm can always satisfy the fairness constraints and achieves high rewards given any $T$, as well as tighter bounds.
\bibliographystyle{named}
\bibliography{ijcai20}
\appendix
\section{Appendix}\label{sec:appendix}
\subsection{Notations and Facts}
Recall that $h_i(t)$ is the number of times that arm $i$ has been pulled at the beginning of round $t$. Recall $\hat{u}_{i}(t) := \frac{\alpha_{i}(t)-1}{h_{i}(t)}=\frac{1}{h_{i}(t)} \sum\limits_{\tau: \tau<t, i \in A(\tau)} X_{i}(t)$ is the empirical mean of arm $i$ at the beginning of round $t$.
Therefore, we have $\alpha_i(t) -1  = \hat{u}_{i}(t) h_i(t) = \hat{u}_{i}(t) (\alpha_{i}(t) + \beta_i(t)-2)$.

For each arm $i\in \mathcal{N}$, we have two events $\mathcal{J}_i(t)$ and $\mathcal{K}_i(t)$ defined as follows:
\begin{equation*}
\begin{aligned}
   &\mathcal{J}_i(t) := \{\theta_i(t)  - u_i > 2\gamma_i(t) \},\\
   &\mathcal{K}_i(t) := \{u_i - \theta_i(t)> 2\gamma_i(t) \}.
\end{aligned}
\end{equation*}
where $\gamma_i(t) := \sqrt{\frac{\ln T}{h_i(t)}}$.

Define $\mathcal{F}_t$ as the history of the plays until time $t$, i.e., $\mathcal{F}_{t}=\left\{i(\tau), r_{i(\tau)}(\tau), \tau=1, \ldots, t\right\}$, where $i(\tau)$ is the arm pulled in round $\tau$.

Recall that the arms pulled by the deterministic oracle in round $t$ are $A^{\prime}(t)$, which is defined by
\begin{equation*}
    A^{\prime}(t) \in \underset{A \subseteq Z(t), |A| \leq m}{\operatorname{argmax}} \sum_{i \in A}\left(\frac{1}{\eta} Q_{i}(t) +  w_{i} u_{i}\right).
\end{equation*}
Let $d_{i}(t):=\mathbf{1}[i=i(t)]$ indicate whether arm $i$ is played by TSCSF-B in round $t$. In the same way, let  $d^\prime_{i}(t)$ indicate whether arm $i$ is played by $A^{\prime}(t)$  in round $t$.

{\bf \noindent  Fact 1  (Chernoff bound)} Let $X_1, \ldots, X_n$ be independent $0$-$1$ random variables such that $\mathbb{E}\left[X_{i}\right]=p_{i}$. Let $X=\frac{1}{n} \sum_{i} X_{i}, \mu=\mathbb{E}[X]=\frac{1}{n} \sum_{i=1}^{n} p_{i}$. Then, for any $0<\lambda<1-\mu$,
\begin{equation*}
    \operatorname{Pr}(X \geq \mu+\lambda) \leq \exp \{-n d(\mu+\lambda, \mu)\},
\end{equation*}
and for any $0<\lambda<\mu$,
\begin{equation*}
    \operatorname{Pr}(X \leq \mu-\lambda) \leq \exp \{-n d(\mu-\lambda, \mu)\},
\end{equation*}
where $d(a, b)=a \ln \frac{a}{b}+(1-a) \ln \frac{(1-a)}{(1-b)}$.

{\bf \noindent  Fact 2  (Hoeffding inequality).} Let $X_1, \ldots, X_n$ be random variables with common range $[0,1]$ and such that $\mathbb{E}\left[X_{t} | X_{1}, \ldots, X_{t-1}\right]=\mu$. Let $S_{n}=X_{1}+\cdots+X_{n}$. Then, for all $a>0$,
\begin{equation*}
    \operatorname{Pr}\left(S_{n} \geq n \mu+a\right) \leq e^{-2 a^{2} / n},
\end{equation*}
\begin{equation*}
    \operatorname{Pr}\left(S_{n} \leq n \mu-a\right) \leq e^{-2 a^{2} / n}.
\end{equation*}

{\bf \noindent  Fact 3  (Relationship between beta and Binomial distributions).} Let $\mathbf{F}_{\alpha, \beta}^{b e t a}(\cdot)$ be the cdf of beta distribution with parameters $\alpha$ and $\beta$, and let $\mathbf{F}_{n, p}^{B}(\cdot)$ be the cdf of Binomial distribution with parameters $n$ and $p$. Then, we have
\begin{equation*}
    \mathbf{F}_{\alpha, \beta}^{b e t a}(y)=1-\mathbf{F}_{\alpha+\beta-1, y}^{B}(\alpha-1)
\end{equation*}
for all positive integers $\alpha$ and $\beta$.

\subsection{Proof of Theorem \ref{thm:regret}}\label{sec:thm52}
\begin{proof}
To prove Theorem \ref{thm:regret}, we first introduce the following lemmas with their proofs in Sec.~\ref{sec:pfoflemmas}.

\begin{lemma}\label{lm:main}
The time-averaged regret of TSCSF-B can be upper bounded by
\begin{equation}\label{eq:astar1}
\resizebox{.9\hsize}{!}{$
\frac{N}{2 \eta}+\frac{1}{T} \left(\underbrace{\sum_{t=0}^{T-1} \mathbb{E}\left[\sum_{i \in A(t)} w_{i}\left(\theta_{i}(t)-u_{i}\right)\right] + \sum_{t=0}^{T-1}\mathbb{E}\left[\sum_{i \in A^{\prime}(t)} w_{i}\left(u_{i}-\theta_{i}(t)\right)\right]}_{C}\right)$,}
\end{equation}
\end{lemma}

\begin{lemma}\label{lm:ft}
For all $i \in \mathcal{N}, t \leq T$, the probability that event ${\mathcal{J}_i(t)}$ happens is upper bounded as follows:
\begin{equation*}
    \Pr\left(\mathcal{J}_i(t)\right) \leq \frac{1}{T^2} + \frac{1}{T^{32}},
\end{equation*}
\end{lemma}

\begin{lemma}\label{lm:kt}
For all $i \in \mathcal{N}$, $t \leq T$, the probability that event ${\mathcal{K}_i(t)}$ happens is upper bounded as follows:
\begin{equation*}
    \Pr\left(\mathcal{K}_i(t)\right) \leq \frac{1}{T^8} + \frac{1}{T^{32}}.
\end{equation*}
\end{lemma}

First, by Lemma \ref{lm:main} we have $R_{\mathrm{TSCSF-B}}(T)$ bounded by
(\ref{eq:astar1}).

\paragraph{Bound $C_1$}
Define event $\overline{\mathcal{J}_i(t)}$ is the complementary event of $\mathcal{J}_i(t)$ as follows:
\begin{equation*}
    \overline{\mathcal{J}_i(t)} := \{\theta_i(t) - u_i \leq 2\gamma_i(t) \}.
\end{equation*}
Then, we can decompose $C_1$ as
\begin{equation*}\label{eq:c1}
\begin{aligned}
    &\underbrace{\sum_{t=0}^{T-1} \mathbb{E}\left[\sum_{i=1}^{N} w_{i}\left(\theta_{i}(t)-u_{i}\right) d_{i}(t) \mathbf{1}[\mathcal{J}_i(t)] \right]}_{B_1} \\
    + &\underbrace{\sum_{t=0}^{T-1} \mathbb{E}\left[\sum_{i=1}^{N} w_{i}\left(\theta_{i}(t)-u_{i}\right) d_{i}(t) \mathbf{1}[\overline{\mathcal{J}_i(t)}] \right]}_{B_2}.
\end{aligned}
\end{equation*}
Since $\theta_{i}(t)-u_{i}\leq 1, d_{i}(t)\leq 1$, $B_1$ is therefore bounded by
\begin{equation*}\label{eq:b2}
\begin{aligned}
    \sum_{t=0}^{T-1} \mathbb{E}\left[\sum_{i=1}^{N} w_{i} \mathbf{1}[{\mathcal{J}_i(t)}] \right] \leq & w_{\max}\sum_{i=1}^{N} \sum_{t=0}^{T-1} \Pr\left({\mathcal{J}_i(t)}\right)\\
    \leq & w_{\max}\sum_{i=1}^{N} \sum_{t=0}^{T-1} \left(\frac{1}{T^2}+\frac{1}{T^{32}}\right)\\
    \leq & w_{\max} N \left(\frac{1}{T}+\frac{1}{T^{31}}\right),
\end{aligned}
\end{equation*}
where the second inequality is due to Lemma \ref{lm:ft}.

Next, we show how to bound $B_2$. Let $\tau_{i}(a)$ be the round when arm $i$ is played for the $a$-th time, i.e., $h_{i}(\tau_{i}(a)) = a-1$.
$B_2$ can be bounded as follows:
\begin{equation}\label{eq:B1}
\begin{aligned}
    &\sum_{t=0}^{T-1} \mathbb{E}\left[\sum_{i=1}^{N} w_{i}\left(\theta_{i}(t)-u_{i}\right) d_{i}(t) \mathbf{1}\left[\overline{\mathcal{J}_i(t)}\right] \right]\\
    =&\sum_{i=1}^{N} \mathbb{E}\left[\sum_{a=1}^{h_{i}(T-1)} \sum_{\tau_i(a)}^{\tau_{i}(a+1)-1} w_{i}\left(\theta_{i}(t)-u_{i}\right) d_{i}(t) \mathbf{1}\left[\overline{\mathcal{J}_i(t)}\right]\right]\\
    \leq& \sum_{i=1}^{N} \mathbb{E}\left[w_{1} + \sum_{a=2}^{h_{i}(T-1)} \sum_{\tau_i(a)}^{\tau_{i}(a+1)-1} w_{i}\left(\theta_{i}(t)-u_{i}\right) d_{i}(t) \mathbf{1}\left[\overline{\mathcal{J}_i(t)}\right]\right]\\
    \leq& w_{\max}\sum_{i=1}^{N} \mathbb{E}\left[1 + 2 \sum_{a=2}^{h_{i}(T-1)} \sum_{\tau_i(a)}^{\tau_{i}(a+1)-1} \gamma_i(t) d_{i}(t) \right]\\
    =& w_{\max}N + 2w_{\max}\sum_{i=1}^{N} \mathbb{E}\left[\sum_{a=2}^{h_{i}(T-1)} \gamma_i(\tau_i(a))\right].
\end{aligned}
\end{equation}
The last term in (\ref{eq:B1}) can be further written as follows:
\begin{equation*}\label{eq:2wb}
    \begin{aligned}
    & 2w_{\max}\sum_{i=1}^{N}  \mathbb{E}\left[ \sum_{a = 2}^{h_i(T-1)} \gamma_i(\tau_i(a))\right]\\
    =& 2w_{\max}\sum_{i=1}^{N} \mathbb{E}\left[ \sum_{a = 2}^{h_i(T-1)} \sqrt{\frac{\ln T}{a-1}}\right]\\
    =& 2w_{\max} \sqrt{\ln T} \sum_{i=1}^{N}  \mathbb{E}\left[ \sum_{a = 2}^{h_i(T-1)} \sqrt{\frac{1}{a-1}}\right] \\
    {\leq} & 2w_{\max} \sqrt{\ln T} \sum_{i=1}^{N}  \mathbb{E}\left[ 1 + \int_{1}^{h_i(T-1)} \sqrt{\frac{1}{x}} dx \right]\\
     {\leq}& 2w_{\max} \sqrt{\ln T} \sum_{i=1}^{N}  \mathbb{E}\left[ \sqrt{h_i(T-1)}\right]\\
     {\leq} & 2w_{\max} \sqrt{\ln T} \sum_{i=1}^{N}   \sqrt{\mathbb{E}\left[h_i(T-1)\right]},
    \end{aligned}
\end{equation*}
where the last inequality is due to Jensen's inequality. Also by Jensen's inequality, we have
\begin{equation*}
    \frac{1}{N}\sum\limits_{i=1}^{N}   \sqrt{\mathbb{E}\left[h_i(T-1)\right]} \leq \sqrt{\frac{1}{N} \sum\limits_{i=1}^{N}\mathbb{E}\left[h_i(T-1)\right]}.
\end{equation*}
Therefore, we can bound $\sum\limits_{i=1}^{N}   \sqrt{\mathbb{E}\left[h_i(T-1)\right]}$ by
\begin{equation*}\label{eq:lb1t}
\begin{aligned}
     \sqrt{N \sum\limits_{i=1}^{N}\mathbb{E}\left[h_i(T-1)\right]}
    \leq \sqrt{NTm},
\end{aligned}
\end{equation*}
where the inequality is due to the fact that at most $m$ arms are selected in each round.
Therefore, we have $B_2$ bounded by
\begin{equation*}\label{eq"b1bound}
    2w_{\max} \sqrt{mNT\ln T} + w_{\max}N.
\end{equation*}

Combining $B_1$ and $B_2$ gives
\begin{equation*}
    C_1 \leq 2w_{\max}\sqrt{mNT\ln T} +  w_{\max}N(1+\frac{1}{T}+\frac{1}{T^{31}}).
\end{equation*}

\paragraph{Bound $C_2$}
Define an event $\overline{\mathcal{K}_i(t)}$ as the complementary event of $\mathcal{K}_i(t)$:
\begin{equation*}
    \overline{\mathcal{K}_i(t)} := \{u_i - \theta_i(t) \leq 2\gamma_i(t)\}.
\end{equation*}
$C_2$ can be decomposed by
\begin{equation*}\label{eq:c2}
\begin{aligned}
    &\underbrace{\sum_{t=0}^{T-1} \mathbb{E}\left[\sum_{i=1}^{N} w_{i}\left(\theta_{i}(t)-u_{i}\right) d^\prime_{i}(t) \mathbf{1}\left[\mathcal{K}_i(t)\right]\right]}_{B_3} \\
    + &\underbrace{\sum_{t=0}^{T-1} \mathbb{E}\left[\sum_{i=1}^{N} w_{i}\left(\theta_{i}(t)-u_{i}\right) d^\prime_{i}(t) \mathbf{1}\left[\overline{\mathcal{K}_i(t)}\right] \right]}_{B_4}.
\end{aligned}
\end{equation*}

Let $\tau^\prime_{i}(a)$ be the round when arm $i$ is played for the $a$-th time by policy $A^\prime(t)$, i.e., $h^\prime_{i}(\tau^\prime_{i}(a)) = a - 1$. Since $\theta_{i}(t)-u_{i}\leq 1$, we can write $B_3$ as follows:
\begin{equation*}\label{eq:b4}
\begin{aligned}
    B_3 \leq& w_{\max} \sum_{i=1}^N \mathbb{E}\left[\sum_{t=1}^{T-1} d^\prime_{i}(t) \mathbf{1}\left[\overline{\mathcal{K}_i(t)}\right] \right] \\
    =& w_{\max} \sum_{i=1}^N \mathbb{E}\left[\sum_{a=0}^{h_i^\prime(T-1)} \mathbf{1}\left[\overline{\mathcal{K}_i(\tau^\prime_{i}(a))}\right] \right] \\
    \leq &  w_{\max} \sum_{i=1}^N \sum_{a=0}^{T-1} \Pr\left(\overline{\mathcal{K}_i(\tau^\prime_{i}(a))}\right)\\
    \stackrel{(\rm a)}{=} &  w_{\max} \sum_{i=1}^N \sum_{a=0}^{T-1} \left(\frac{1}{(\tau^\prime_{i}(a))^2} + \frac{1}{(\tau^\prime_{i}(a))^4} \right)\\
    \leq &  w_{\max} \sum_{i=1}^N \sum_{t=0}^{T-1} \left(\frac{1}{T^8} + \frac{1}{T^{32}} \right)\\
    \leq & w_{\max}N \left(\frac{1}{T^7} + \frac{1}{T^{31}} \right).
\end{aligned}
\end{equation*}
where $\rm{(a)}$ is due to Lemma \ref{lm:kt}.

$B_4$ can be bounded in a similar way as $B_2$:
\begin{equation*}\label{eq:b3}
\begin{aligned}
    B_4 \leq& w_{\max}N + 2w_{\max} \sum_{i=1}^{N} \mathbb{E}\left[\sum_{a=2}^{h^\prime_i(T-1)}\sqrt{\frac{ \ln T}{a-1}} \right]\\
    \leq & w_{\max}N  + 2w_{\max} \sqrt{\ln T} \sum_{i=1}^{N} \mathbb{E}\left[\sum_{a=2}^{h^\prime_i(T-1)}\sqrt{\frac{1}{a-1}} \right]\\
    \leq & 2w_{\max} \sqrt{mNT\ln T} + w_{\max}N .
\end{aligned}
\end{equation*}

Therefore, we have $C_2$ bounded by
\begin{equation*}\label{eq:c21}
    2w_{\max}\sqrt{mNT\ln T} + w_{\max}N \left(1+\frac{1}{T^7} + \frac{1}{T^{31}} \right).
\end{equation*}
Combining $C_1$ and $C_2$, when $T>1$, we have $C$ upper bounded by
\begin{equation}
    4w_{\max}\sqrt{mNT\ln T} + 2.51 w_{\max}N.
\end{equation}
\end{proof}

\subsection{Proof of Lemmas}\label{sec:pfoflemmas}

\subsubsection{Proof of Lemma \ref{lm:main}}

\begin{proof}
The proof for Lemma \ref{lm:main} is similar to Theorem 1 in [Li et al., 2019]. However, the arms selection (defined in Eq.~(\ref{eq:solution}) of our paper) is different from LFG~[Li et al., 2019] (we put $\eta$ together with the virtual queue). We first consider the Lyapunov drift function:
\begin{equation*}
    L(\mathbf{Q}(t)) := \frac{1}{2} \sum_{i=1}^{N} Q_{i}^{2}(t).
\end{equation*}
Recall the regret definition as follows:
\begin{equation*}
\begin{aligned}
    \mathcal{R}(T) &:= \mathbb{E}\left[\frac{1}{T} \sum_{t=0}^{T-1} {\left(\sum_{i \in A^*(t)} w_{i} X_{i}(t) - \sum_{i \in A(t)} w_{i} X_{i}(t)\right)}\right]\\
    & = \frac{1}{T} \sum_{t=0}^{T-1}\mathbb{E}\left[\sum_{i \in A^*(t)} w_{i} u_{i}(t) - \sum_{i \in A(t)} w_{i} u_{i}(t)\right]\\
    & = \frac{1}{T} \sum_{t=0}^{T-1}\mathbb{E}\left[\underbrace{\sum_{i = 1}^N w_{i} u_{i}(t) d^*_i(t)- \sum_{i = 1}^N w_{i} u_{i}(t) d_i(t)}_{\Delta r_i(t)} \right].,
\end{aligned}
\end{equation*}
where $d^*_i(t) = \mathbf{1}\left[i \in A^*_i(t) \right]$ and $d_i(t) = \mathbf{1}\left[i \in A_i(t) \right]$.

Then, the drift-plus-regret is given by
\begin{equation}
\begin{aligned}
    &L(\mathbf{Q}(t+1))-L(\mathbf{Q}(t))+\eta \Delta r_i(t) \\
    &=\frac{1}{2} \sum\limits_{i=1}^{N} Q_{i}^{2}(t+1)-\frac{1}{2} \sum\limits_{i=1}^{N} Q_{i}^{2}(t)+\eta \Delta r_i(t) \\
    &\stackrel{(\rm a)} {=}\frac{1}{2}\sum\limits_{i=1}^{N}\left(Q_{i}(t)+k_{i}-d_{i}(t)\right)^{2}-\frac{1}{2} \sum\limits_{i=1}^{N} Q_{i}^{2}(t)+\eta \Delta r_i(t) \\
    &=\frac{1}{2} \sum\limits_{i=1}^{N}\left(k_{i}-d_{i}(t)\right)^{2}+\sum\limits_{i=1}^{N}\left(k_{i}-d_{i}(t)\right) Q_{i}(t)+\eta \Delta r_i(t) \\
    &\stackrel{(\rm b)} {\leq} \frac{N}{2}+\sum\limits_{i=1}^{N} k_{i} Q_{i}(t)-\sum\limits_{i=1}^{N} d_{i}(t) Q_{i}(t)\\
    &+\eta \sum\limits_{i=1}^{N} w_{i} u_{i} d_{i}^{*}(t)-\eta \sum\limits_{i=1}^{N} w_{i} u_{i} d_{i}(t) \\
    &= \frac{N}{2}+\sum\limits_{i=1}^{N}\left(Q_{i}(t)+\eta w_{i} u_{i}\right)\left(d_{i}^{*}(t)-d_{i}(t)\right) \\
    &+\sum\limits_{i=1}^{N} Q_{i}(t)\left(k_{i}-d_{i}^{*}(t)\right),
\end{aligned}
\end{equation}
where $\rm{(a)}$ is due to the queue evolution equation~(defined in (5) of our paper), and $\rm{(b)}$ is due to facts that $k_{i}-d_{i}(t)\leq 1$ and $\Delta r_i(t) > 0$

We can bound the expected drift-plus-regret as
\begin{equation}\label{eq:dpr}
\begin{aligned}
     &\mathbb{E}[L(\mathbf{Q}(t+1))-L(\mathbf{Q}(t))+\eta \Delta r_i(t)]\\
     &\leq \frac{N}{2}+\sum\limits_{i=1}^{N} \mathbb{E}\left[\left(Q_{i}(t)+\eta w_{i} u_{i}\right)\left(d_{i}^{*}(t)-d_{i}(t)\right)\right]\\
     &+\sum\limits_{i=1}^{N} \mathbb{E}\left[Q_{i}(t)\left(k_{i}-d_{i}^{*}(t)\right)\right] \\
     &\leq \frac{N}{2}+\mathbb{E}[\sum\limits_{i=1}^{N}\left(Q_{i}(t)+\eta w_{i} u_{i}\right)\left(d_{i}^{*}(t)-d_{i}(t)\right)],
\end{aligned}
\end{equation}
where the last inequality is due to $\mathbb{E}\left[d_{i}^{*}(t)\right] \geq k_{i}$.

Summing (\ref{eq:dpr}) for all
$t\in\{0,\ldots,T-1\}$, and
dividing both sides of the inequality by $T\eta$, we have
\begin{equation*}
\begin{aligned}
    &\frac{1}{T \eta} \mathbb{E}[L(\mathbf{Q}(T))-L(\mathbf{Q}(0))]+\frac{1}{T} \sum\limits_{t=0}^{T-1} \mathbb{E}[\Delta r_i(t)] \\
    &\leq \frac{N}{2 \eta}+\frac{1}{T} \mathbb{E}[\underbrace{\sum\limits_{i=1}^{N}  \left(\frac{1}{\eta}Q_{i}(t)+ w_{i} u_{i}\right)\left(d_{i}^{*}(t)-d_{i}(t)\right)}_{C(t)}].
\end{aligned}
\end{equation*}

Since $L(Q(T)) \geq 0$ and $L(Q(0)) = 0$, we have
\begin{equation}\label{eq:regretbound}
    \mathcal{R}(T) = \frac{1}{T} \sum_{t=0}^{T-1} \mathbb{E}[\Delta r_i(t)] \leq \frac{N}{2 \eta}+\frac{1}{\eta T} \sum_{t=0}^{T-1} \mathbb{E}\left[C(t)\right].
\end{equation}

Recall that in each round $t$, the TSCSF-B algorithm chooses
arms $A(t)$ according to the follows:
\begin{equation*}
    A(t) \in \underset{A \subseteq {Z}(t), |A| \leq m}{\operatorname{argmax}} \sum_{i \in S}\left(\frac{1}{\eta}Q_{i}(t)+ w_{i} \theta_{i}(t)\right),
\end{equation*}
and the oracle (see the sketch proof for Theorem 2) chooses arms $A^\prime (t)$ as follows:
\begin{equation}\label{eq:1eeeeee}
    A^{\prime}(t) \in \underset{A \subseteq Z(t), |A| \leq m}{\operatorname{argmax}} \sum_{i \in A}\left(\frac{1}{\eta}Q_{i}(t)+w_{i} u_{i}(t)\right).
\end{equation}

Therefore, we have
\begin{equation}\label{eq:1eee}
    \sum_{i \in A(t)}\left(\frac{1}{\eta}Q_{i}(t)+ w_{i} \theta_{i}(t)\right) \geq \sum_{i \in A^{\prime}(t)}\left(\frac{1}{\eta} Q_{i}(t) + w_{i} \theta_{i}(t)\right).
\end{equation}

$C(t)$ can be bounded as follows:
\begin{equation}\label{eq:ct}
\resizebox{.9\hsize}{!}{$
    \begin{aligned}
    C(t)=& \sum_{i=1}^{N}\left(\frac{1}{\eta} Q_{i}(t)+ w_{i} u_{i}\right)\left(d_{i}^{*}(t)-d_{i}(t)\right) \\
    =& \sum_{i \in A^{*}(t)}\left(\frac{1}{\eta} Q_{i}(t)+ w_{i} u_{i}\right)-\sum_{i \in A(t)}\left(\frac{1}{\eta} Q_{i}(t)+w_{i} u_{i}\right) \\
     \stackrel{(\rm a)}{\leq} & \sum_{i \in A^{\prime}(t)}\left(\frac{1}{\eta}Q_{i}(t)+ w_{i} u_{i}\right)-\sum_{i \in A(t)}\left(\frac{1}{\eta}Q_{i}(t)+ w_{i} u_{i}\right) \\
     \stackrel{(\rm b)}\leq & \sum_{i \in A^{\prime}(t)}\left(\frac{1}{\eta}Q_{i}(t)+ w_{i} u_{i}\right)-\sum_{i \in A(t)}\left(\frac{1}{\eta}Q_{i}(t)+ w_{i} u_{i}\right) \\
     +&\sum_{i \in A(t)}\left(\frac{1}{\eta}Q_{i}(t)+ w_{i} \theta_{i}(t)\right) -\sum_{i \in A^{\prime}(t)}\left(\frac{1}{\eta}Q_{i}(t)+ w_{i} \theta_{i}(t)\right) \\
     =& {\sum_{i \in A(t)} w_{i}\left(\theta_{i}(t)-u_{i}\right)} + {\sum_{i \in A^{\prime}(t)} w_{i}\left(u_{i}-\theta_{i}(t)\right)},
     \end{aligned}$}
\end{equation}
where $\rm{(a)}$ is due to the definition of $A^\prime(t)$ in (\ref{eq:1eeeeee}) and $\rm{(b)}$ is due to (\ref{eq:1eee}).

Substituting (\ref{eq:ct}) into (\ref{eq:regretbound}) concludes the proof.
\end{proof}

\subsubsection{Proof of Lemma  \ref{lm:ft}}
\begin{proof}
Let event $\mathcal{A}_i(t) := \left\{\hat{u}_i(t) - u_i \leq 4\gamma_i(t)\right\}$. Then, we have
\begin{equation}\label{eq:prnf}
\begin{aligned}
   \Pr\left(\mathcal{J}_i(t)\right) =& \Pr\left({\mathcal{J}_i(t)}|\mathcal{A}_i(t) \right) \Pr(\mathcal{A}_i(t)) \\  + & \Pr\left({\mathcal{J}_i(t)}|\overline{\mathcal{A}_i(t)} \right) \Pr(\overline{\mathcal{A}_i(t)})\\
   \leq&\Pr\left({\mathcal{J}_i(t)}|\mathcal{A}_i(t)\right) + \Pr\left(\overline{\mathcal{A}_i(t)}\right).
\end{aligned}
\end{equation}
We can bound $\Pr\left({\mathcal{J}_i(t)}|\mathcal{A}_i(t)\right)$ as follows:
\begin{equation}\label{eq:pra1}
\begin{aligned}
    \Pr\left({\mathcal{J}_i(t)}|\mathcal{A}_i(t)\right) =& \Pr\left(\theta_i(t)>u_i+2\gamma_i(t)|\mathcal{A}_i(t)\right)\\
    \stackrel{(\rm a)}{\leq} & \Pr\left(\theta_i(t)>\hat{u}_i(t) - 2\gamma_i(t)\right)\\
    = & \mathbb{E}\left[{\Pr\left(\theta_i(t)>\hat{u}_i(t)  - 2\gamma_i(t) | \mathcal{F}_{t-1} \right)}\right],
\end{aligned}
\end{equation}
where $\rm{(a)}$ is due to the fact that $u_i + 2\gamma_i(t) \geq \hat{u}_i(t) - 2\gamma_i(t)$ conditioned on $\mathcal{A}_i(t)$ happens.

Since given $\mathcal{F}_{t-1}$, $\hat{u}_i(t)$ and $\gamma_i(t)$ are determined, we have
\begin{equation}\label{eq:pra3}
\begin{aligned}
     &\Pr\left(\theta_i(t)>\hat{u}_i(t) - 2\gamma_i(t) | \mathcal{F}_{t-1} \right)\\
    =& 1-\mathbf{F}_{\alpha_i(t),\beta_i(t)}^{beta}\left(\hat{u}_i(t) - 2\gamma_i(t)\right)  \\
    \stackrel{(\rm b)}{=}& \mathbf{F}_{\alpha_i(t)+\beta_i(t)-1,\hat{u}_i(t)-2\gamma_i(t)}^{B}\left(\alpha_i(t)-1\right), \\
\end{aligned}
\end{equation}
where $\rm{(b)}$ is due to Fact 3 for the relationship between beta and Binomial distributions. Since $\alpha_i(t)-1 = \hat{u}_i(t)(\alpha_i(t)+\beta_i(t)-2)<\hat{u}_i(t)(\alpha_i(t)+\beta_i(t)-1)$, (\ref{eq:pra3}) can be further bounded by

\begin{equation}\label{eq:pra2}
    \begin{aligned}
    &\mathbf{F}_{\alpha_i(t)+\beta_i(t)-1,\hat{u}_i(t)-2\gamma_i(t)}^{B}\left(\hat{u}_i(t)(\alpha_i(t)+\beta_i(t)-1)\right) \\
    \stackrel{(\rm a)}{\leq} & e^{-(\alpha_i(t)+\beta_i(t)-1)d(\hat{u}_i(t), \hat{u}_i(t)-2\gamma_i(t))}\\
    \stackrel{(\rm b)}{\leq} & e^{-h_i(t)\frac{|\hat{u}_i(t)-\hat{u}_i(t)+2\gamma_i(t)|^2}{2}}\\
    =& e^{-h_i(t){2\gamma^2_i(t)}}\\
    =& \frac{1}{T^2},
    \end{aligned}
\end{equation}
where $\rm{(a)}$ is due to Fact 1 by the Chernoff bound and $\rm{(b)}$ is due to the fact that $d(a,b)\geq \frac{|a-b|^2}{2}$.
Substituting (\ref{eq:pra2}) and (\ref{eq:pra3}) into (\ref{eq:pra1}), we have
\begin{equation}\label{eq:pra4}
\begin{aligned}
    \Pr\left({\mathcal{J}_i(t)}|\mathcal{A}_i(t)\right)\leq \frac{1}{T^2}.
\end{aligned}
\end{equation}

On the other hand, $\Pr\left(\overline{\mathcal{A}_i(t)}\right)$ can be written as
\begin{equation*}\label{eq:prna1}
    \Pr\left(\overline{\mathcal{A}_i(t)}\right) = \mathbb{E}\left[\Pr\left(\overline{\mathcal{A}_i(t)} | \mathcal{F}_{t-1}\right)\right].
\end{equation*}
Given $\mathcal{F}_{t-1}$,  $\hat{u}_i(t)$ and $\gamma_i(t)$ are determined. Then
by Fact 2, we have
\begin{equation*}\label{eq:prna}
    \begin{aligned}
        \Pr\left(\overline{\mathcal{A}_i(t)}| \mathcal{F}_{t-1}\right) =& \Pr(\hat{u}_i(t) - u_i > 4\gamma_i(t))\\
        \leq& e^{-32(\gamma_i(t))^2h_i(t)}=\frac{1}{T^{32}}.
    \end{aligned}
\end{equation*}
Therefore,  we have
\begin{equation}\label{eq:prna3}
    \Pr\left(\overline{\mathcal{A}_i(t)}\right) \leq \frac{1}{T^{32}}.
\end{equation}

Substituting (\ref{eq:prna3}), (\ref{eq:pra4}) into (\ref{eq:prnf}) concludes the proof.
\end{proof}

\subsubsection{Proof of Lemma \ref{lm:kt}}
\begin{proof}
Define event $\mathcal{G}_i(t)$ as
\begin{equation*}
    \mathcal{G}_i(t) := \{u_i-\hat{u}_i(t) \leq 4\gamma_i(t)\}.
\end{equation*}
We can decompose $\Pr\left({\mathcal{K}_i(t)}\right)$ as follows
\begin{equation}\label{eq:prnk}
\begin{aligned}
    &\Pr\left({\mathcal{K}_i(t)}|\mathcal{G}_i(t) \right) \Pr(\mathcal{G}_i(t)) + \Pr\left({\mathcal{K}_i(t)}|\overline{\mathcal{G}_i(t)} \right) \Pr(\overline{\mathcal{G}_i(t)})\\
   \leq&\Pr\left({\mathcal{K}_i(t)}|\mathcal{G}_i(t)\right) + \Pr\left(\overline{\mathcal{G}_i(t)}\right).
\end{aligned}
\end{equation}
For each arm $i$, since $u_i-2\gamma_i(t)\leq \hat{u}_i(t)+2{\gamma_i(t)}$ when $\mathcal{G}_i(t)$ happens, we have $\Pr\left({\mathcal{K}_i(t)}|\mathcal{G}_i(t)\right)$ bounded by
\begin{equation}\label{eq:pr1}
\begin{aligned}
     & \Pr\left(\theta_i(t)<\hat{u}_i+2\gamma_i(t)\right)=  \mathbb{E}\left[\Pr\left(\theta_i(t)<\hat{u}_i+2\gamma_i(t) | \mathcal{F}_{t-1}  \right) \right]
\end{aligned}
\end{equation}
Given $\mathcal{F}_{t-1}$, $\hat{u}_i(t)$ and $\gamma_i(t)$ are determined. Then, we can write $\Pr\left(\theta_i(t)<\hat{u}_i(t)+2\gamma_i(t) | \mathcal{F}_{t-1}\right)$ as
\begin{equation}\label{eq:pr11}
\begin{aligned}
     &\mathbf{F}_{\alpha_i(t), \beta_i(t)}^{beta}(\hat{u}_i(t)+2{\gamma_i(t)})\\
     = & 1 -  \mathbf{F}_{\alpha_i(t) + \beta_i(t)-1, \hat{u}_i(t)+2{\gamma_i(t)}}^{B}(\alpha_i(t)-1) \\
      \stackrel{(\rm b)} {=} & 1 -  \mathbf{F}_{h_i(t) + 1, \hat{u}_i(t)+2{\gamma_i(t)}}^{B}(\hat{u}_i(t) h_i(t)),
\end{aligned}
\end{equation}
where $\rm{(b)}$ is due to $h_i(t) = \alpha_i(t) + \beta_i(t)-2$ and $\alpha_i(t)-1 = \hat{u}_i(t)h_i(t)$.

Let $Y_{j}$ be an outcome of the $j$-th Bernoulli trial with mean value $\hat{u}_i(t)+2{\gamma_i(t)}$. $S_i := \sum\limits_{j=1}^{h_i(t)+1} Y_j$ is a random variable generated from Binomial distribution with parameters $h_i(t) + 1$~(number of trials) and $\hat{u}_i(t)+2{\gamma_i(t)}$~(mean value). Then, we have
\begin{equation}\label{eq:Fb1m}
\begin{aligned}
    F_{h_i(t) + 1, \hat{u}_i(t)+2{\gamma_i(t)}}^{B}(\hat{u}_i(t) h_i(t)) =& \Pr\left(S_i \leq \hat{u}_i(t) h_i(t)\right) \\
    = & 1 - \Pr\left(S_i > \hat{u}_i(t) h_i(t)\right).
\end{aligned}
\end{equation}
Substituting (\ref{eq:Fb1m}) into (\ref{eq:pr11}), we have
\begin{equation}\label{eq:fb2m}
    \mathbf{F}_{\alpha_i(t), \beta_i(t)}^{beta}(\hat{u}_i(t)+2{\gamma_i(t)})=\Pr\left(S_i > \hat{u}_i h_i(t)\right).
\end{equation}
We can rewrite $\Pr\left(S_i > \hat{u}_i h_i(t)\right)$ as
\begin{equation*}\label{eq:lema22}
\begin{aligned}
\Pr\left(S_i > n \mathbb{E}[Y_j]+\delta_i\right),
\end{aligned}
\end{equation*}
where $n=h_i(t)+1$, $\mathbb{E}[Y_j]=\hat{u}_i(t)+2{\gamma_i(t)}$, and $\delta_i = \hat{u}_i(t) h_i(t)-(\hat{u}_i(t)+2{\gamma_i(t)})(h_i(t)+1)$.

Therefore, according to the Hoeffding inequality (Fact 2), we have
\begin{equation}\label{eq:g1}
\begin{aligned}
    &\Pr\left(S_i > \hat{u}_i(t) h_i(t)\right)\\
    \leq& e^{-2(\delta_i)^2/n^2}\\
    =&e^{-2 \left(\hat{u}_i(t) h_i(t) - (\hat{u}_i(t)+2{\gamma_i(t)})(h_i(t)+1)\right)^2/(h_i(t)+1)}\\
    =& e^{-2(\hat{u}^2_i(t) +  4 \hat{u}_i(t) \gamma_i(t)  (h_i(t)+1) + {\gamma_i(t)^2}(h_i(t)+1)^2)/(h_i(t)+1)} \\
    \leq & e^{-2( 4 \hat{u}_i(t) \gamma_i(t) + 4 {\gamma_i(t)^2}(h_i(t)+1))}\\
    = & e^{-8 \hat{u}_i(t) \gamma_i(t) -8{\gamma_i(t)^2}(h_i(t)+1)}\\
    \leq & e^{ -8{\varepsilon^2_i(t)}h_i(t)}\\
    = & \frac{1}{T^8}.
\end{aligned}
\end{equation}

Substituting (\ref{eq:g1}) into (\ref{eq:fb2m}), we have
\begin{equation*}
    \Pr\left(\theta_i(t)<\hat{u}_i(t)-\gamma_i(t) | \mathcal{F}_{t-1}  \right) \leq \frac{1}{T^8}.
\end{equation*}
Therefore,
\begin{equation*}\label{eq:t22}
   \Pr\left({\mathcal{K}_i(t)}|\mathcal{G}_i(t)\right) \leq  \frac{1}{T^8}.
\end{equation*}
On the other hand, $\Pr \left( \overline{\mathcal{G}_i(t)} \right) = \mathbb{E}\left[\Pr \left( \overline{\mathcal{G}_i(t)} | \mathcal{F}_{t-1} \right) \right]$. Given $\mathcal{F}_{t-1}$, $\hat{u}_i(t)$ and $\epsilon_i(t)$ are determined, and $\Pr \left( \overline{\mathcal{G}_i(t)} | \mathcal{F}_{t-1} \right)$ can be bounded by the Hoeffding inequality (Fact 2) :
\begin{equation*}
\begin{aligned}
    \Pr \left( \overline{\mathcal{G}_i(t)}| \mathcal{F}_{t-1} \right)=& \Pr \left(\hat{u}_i(t) < u_i - 4{\gamma_i(t)} | \mathcal{F}_{t-1} \right)\\
    \leq & e^{-2 ({4\gamma_i(t)})^2 h_i(t)}\\
    =& \frac{1}{T^{32}}.
\end{aligned}
\end{equation*}
Therefore, $\Pr \left( \overline{\mathcal{G}_i(t)} \right)$ can be bounded by $\frac{1}{T^{32}}$.
\end{proof}

\end{document}